\documentclass{article}

\usepackage{microtype}
\usepackage{graphicx}
\graphicspath{{figures/}}
\usepackage{subfigure}
\usepackage{amsmath}
\usepackage{tikz}
\usepackage{natbib}
\usepackage{graphicx}
\usepackage{hyperref}
\usepackage{mystyle}
\usepackage{amsthm}
\usepackage[english]{babel}
\usepackage{comment}
\usepackage{enumitem}
\usepackage[colorinlistoftodos,prependcaption,textsize=tiny]{todonotes}
\usepackage{multicol}
\usepackage{booktabs} %
\usepackage{adjustbox}
\usepackage[belowskip=-1pt,aboveskip=1pt]{caption}
\usepackage{amsmath}
\usepackage{listings}
\usepackage{xargs}  
\usepackage{amssymb}%
\usepackage{pifont}%
\usepackage{algorithmic}
\newcommand{\cmark}{\ding{51}}%
\newcommand{\xmark}{\ding{55}}%
\usepackage{hyperref}
\usepackage[frozencache]{minted}
\newtheorem{theorem}{Theorem}

\newtheorem{lemma}{Lemma}
\newtheorem{proposition}{Proposition}
\usepackage[accepted]{icml2021}

\icmltitlerunning{Momentum Residual Neural Networks}

\begin{document}

\twocolumn[
\icmltitle{Momentum Residual Neural Networks}

\icmlsetsymbol{equal}{*}

\begin{icmlauthorlist}
\icmlauthor{Michael E. Sander}{ens,cnrs}
\icmlauthor{Pierre Ablin}{ens,cnrs}
\icmlauthor{Mathieu Blondel}{google}
\icmlauthor{Gabriel Peyr\'e}{ens,cnrs}
\end{icmlauthorlist}

\icmlaffiliation{ens}{Ecole Normale Sup\'erieure, DMA, Paris, France}
\icmlaffiliation{cnrs}{CNRS, France}
\icmlaffiliation{google}{Google Research, Brain team}

\icmlcorrespondingauthor{Michael Sander}{michael.sander@ens.fr}
\icmlcorrespondingauthor{Pierre Ablin}{pierre.ablin@ens.fr}
\icmlcorrespondingauthor{Mathieu Blondel}{mblondel@google.com}
\icmlcorrespondingauthor{Gabriel Peyr\'e}{gabriel.peyre@ens.fr}

\icmlkeywords{Machine Learning, ICML}

\vskip 0.3in
]

\printAffiliationsAndNotice{} 

\begin{abstract}
The training of deep residual neural networks (ResNets) with backpropagation has a memory cost that increases linearly with respect to the depth of the network. 
A way to circumvent this issue is to use reversible architectures.
In this paper, we propose to change the forward rule of a ResNet by adding a momentum term. The resulting networks, momentum residual neural networks (Momentum ResNets),
are invertible.
Unlike previous invertible architectures, they can be used as a drop-in replacement for any existing ResNet block.
We show that Momentum ResNets can be interpreted in the infinitesimal step size regime as second-order ordinary differential equations (ODEs) and exactly characterize how adding momentum progressively increases the representation capabilities of Momentum ResNets: they can learn any linear mapping up to a multiplicative factor, while ResNets cannot.
In a learning to optimize setting, 
where convergence to a fixed point is required, we show theoretically and empirically that our method succeeds while existing invertible architectures fail.  
We show on CIFAR and ImageNet that Momentum ResNets have the same accuracy as ResNets, while having a much smaller memory footprint, and show that pre-trained Momentum ResNets are promising for fine-tuning models.
\end{abstract}

\vspace{-1em}
\section{Introduction}
\label{submission}

\paragraph{Problem setup.}

As a particular instance of deep learning \citep{cite-key,Goodfellow-et-al-2016}, residual neural networks \citep[ResNets]{he2015deep} have achieved great empirical successes due to extremely deep representations and their extensions keep on outperforming state of the art on real data sets \citep{alex2019big,touvron2020fixing}.
Most of deep learning tasks involve graphics processing units (GPUs), where memory is a practical bottleneck in several situations \citep{wang2018superneurons,peng2017large,zhu2017unpaired}.
Indeed, backpropagation, used for optimizing deep architectures, requires to store values (activations) at each layer during the evaluation of the network (forward pass). Thus, the depth of deep architectures is constrained by the amount of available memory. The main goal of this paper is to explore the properties of a new model, Momentum ResNets, that circumvent these memory issues by being invertible: the activations at layer $n$ is recovered exactly from activations at layer $n+1$.
This network relies on a modification of the ResNet's forward rule which makes it exactly invertible in practice. 
Instead of considering the feedforward relation for a ResNet (residual building block)
\begin{equation}\label{eq:ResNet}
    x_{n+1} = x_n + f(x_n,\theta_n),
\end{equation}
we define its momentum counterpart, which iterates %
\begin{equation}\label{eq:Momentum ResNet}
\left\{
\begin{array}{r@{\hspace{1mm}}l}
v_{n+1} & = \gamma v_n + (1-\gamma) f(x_n,\theta_n)\\ 
x_{n+1} & = x_n + v_{n+1}, 
\end{array}
\right.
\end{equation}
where $f$ is a parameterized function, $v$ is a velocity term and $\gamma \in [0, 1]$ is a momentum term. This radically changes the dynamics of the network, as shown in the following figure.
\begin{figure}[H]
\centering
\includegraphics[width=0.7\columnwidth]{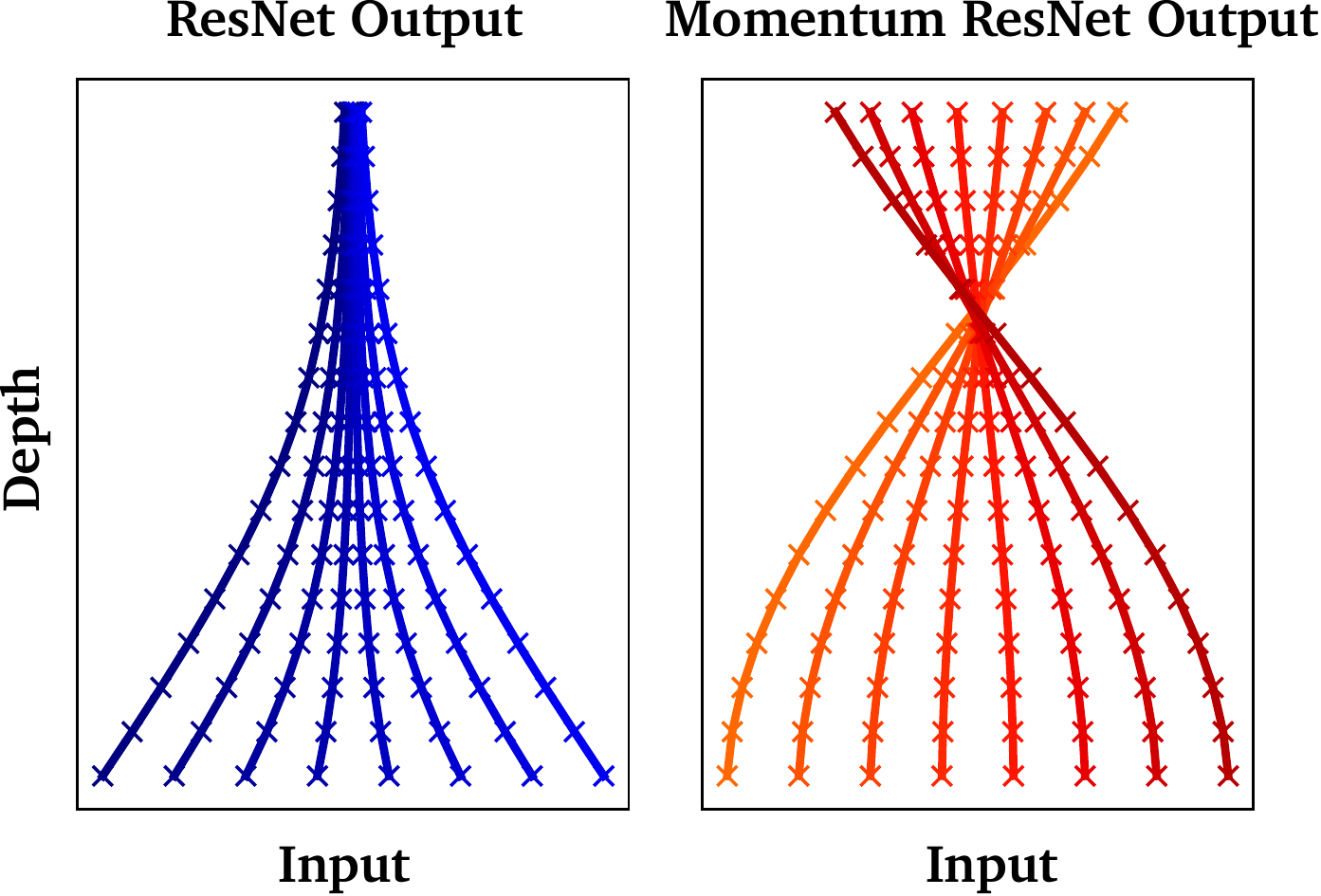} 
\caption{\textbf{Comparison of the dynamics} of a ResNet (left) and a Momentum ResNet with $\gamma = 0.9$ (right) with \textbf{tied weights} between layers, $\th_n=\th$ for all $n$. The evolution of the activations at each layer is shown (depth 15). Models try to learn the mapping $x \mapsto -x^{3}$ in $\RR$. The ResNet fails (the iterations approximate the solution of a first-order ODE, for which trajectories don’t cross, cf.\ Picard-Lindelof theorem)
while the Momentum ResNet leverages the changes in velocity to model more complex dynamics.}\label{fig:Res_Mom}
\vspace{-1em}
\end{figure}
In contrast with existing reversible models, Momentum ResNets can be integrated seamlessly in any deep architecture which uses residual blocks as building blocks (cf.\ in Section~\ref{section:mom_net}).
\vspace{-1em}
\paragraph{Contributions.}

We  introduce momentum residual neural networks (Momentum ResNets), a new deep model that relies on a simple modification of the ResNet forward rule and which, without any constraint on its architecture, is perfectly invertible. We show that the memory requirement of Momentum ResNets is arbitrarily reduced by changing the momentum term $\gamma$ (Section~\ref{section:memory_cost}), and show that they can be used as a drop-in replacement for traditional ResNets.

On the theoretical side,  we show that Momentum ResNets are easily used in the learning to optimize setting, where other reversible models fail to converge (Section~\ref{sec:role_momentum}). We also investigate the approximation capabilities of Momentum ResNets, seen in the continuous limit as second-order ODEs (Section~\ref{section:representation_capabilities}). We first show in Proposition~\ref{prop:bigger_set} that Momentum ResNets can represent a strictly larger class of functions than first-order neural ODEs. Then, we give more detailed insights by studying the linear case, where we formally prove in Theorem~\ref{th:representability} that Momentum ResNets with linear residual functions have universal approximation capabilities, and precisely quantify how the set of representable mappings for such models grows as the momentum term $\gamma$ increases. This theoretical result is a first step towards a theoretical analysis of representation capabilities of Momentum ResNets.
 
 Our last contribution is the experimental validation of Momentum ResNets on various learning tasks. We first show that Momentum ResNets separate point clouds that ResNets fail to separate (Section~\ref{sec:point_clouds}). We also show on image datasets (CIFAR-10, CIFAR-100, ImageNet) that Momentum ResNets have similar accuracy as ResNets, with a smaller memory cost (Section~\ref{sec:experiments_images}). We also show that parameters of a pre-trained model are easily transferred to a Momentum ResNet which achieves comparable accuracy in only few epochs of training. We argue that this way to obtain pre-trained Momentum ResNets is of major importance for fine-tuning a network on new data for which memory storage is a bottleneck.  We provide a Pytorch package with a method that takes a torchvision ResNet model and returns its Momentum counterpart that achieves similar accuracy with very little refit.  We also experimentally validate our theoretical findings in the learning to optimize setting, by confirming that Momentum ResNets perform better than RevNets \citep{gomez2017reversible}. 
 Our code is available at \url{https://github.com/michaelsdr/momentumnet}.

\section{Background and previous works.}

\paragraph{Backpropagation.}
\textit{Backpropagation} is the method of choice to compute the gradient of a scalar-valued function.
It operates using the chain rule with a backward traversal of the computational graph \citep{Computational_Graphs}. It is also known as reverse-mode automatic differentiation \citep{baydin2015automatic, 1986Natur.323..533R,verma2000introduction, doi:10.1137/1.9780898717761}. The computational cost is similar to the one of evaluating the function itself. The only way to back-propagate gradients through a neural architecture without further assumptions is to store all the intermediate activations during the forward pass. This is the method used in common deep learning libraries such as Pytorch~\citep{paszke2017automatic}, Tensorflow~\citep{tensorflow2015-whitepaper} and JAX~\citep{jacobsen2018irevnet}. A common way to reduce this memory storage is to use checkpointing: activations are only stored at some steps and the others are recomputed between these check-points as they become needed in the backward pass~(e.g., \citet{martens2012training}).
\vspace{-1em}
\paragraph{Reversible architectures.}
However, models that allow backpropagation without storing any activations have recently been developed. They are based on two kinds of approaches.
The first is \textit{discrete} and relies on finding ways to easily invert the rule linking activation $n$ to activation $n+1$ \citep{gomez2017reversible,chang2017reversible,Haber_2017,jacobsen2018irevnet, behrmann2019invertible}.
In this way, it is possible to recompute the activations \emph{on the fly} during the backward pass: activations do not have to be stored.
However, these methods either rely on restricted architectures where there is no straightforward way to transfer a well performing non-reversible model into a reversible one, or do not offer a fast inversion scheme when recomputing activations backward. In contrast, our proposal can be applied to any existing ResNet and is easily inverted.
The second kind of approach is \textit{continuous} and
relies on ordinary differential equations (ODEs), where ResNets are interpreted as continuous dynamical systems \citep{E_2017,chen2018neural,dupont2019augmented,sun2018stochastic,E_2018,lu2017finite,ruthotto2018deep}. This allows one to import theoretical and numerical advances from ODEs to deep learning.
These models are often called neural ODEs~\citep{chen2018neural} and can be trained by using an adjoint sensitivity method \citep{Pontryagin:234445}, solving ODEs backward in time. This strategy avoids performing reverse-mode automatic differentiation 
through the operations of the ODE solver and leads to a $O (1)$ memory footprint.
However, defining the neural ODE counterpart of an existing residual architecture is not straightforward: optimizing ODE blocks is an infinite dimensional problem requiring a non-trivial time discretization, and the performances of neural ODEs depend on the numerical integrator for the ODE \citep{gusak2020towards}.
In  addition, ODEs cannot always be numerically reversed, because of stability issues: numerical errors can occur and accumulate when a system is run backwards \citep{gholami2019anode, dupont2019augmented}. Thus, in practice, neural ODEs are seldom used in standard deep learning settings.  Nevertheless, recent works \citep{zhang2019anodev2, queiruga2020continuous} incorporate ODE blocks in neural architectures to achieve comparable accuracies to ResNets on CIFAR.

\vspace{-0.5em}

\vspace{-0.5em}
 \paragraph{Representation capabilities.}

Studying the representation capabilities of such models is also important, as it gives insights regarding their performance on real world data. %
  It is well-known that a single residual block has universal approximation capabilities \citep{Cybenkot2006ApproximationBS}, meaning that on a compact set any continuous function can be uniformly approximated with a one-layer feedforward fully-connected neural network. However, neural ODEs have limited representation capabilities. 
 \citet{dupont2019augmented} propose to lift points in higher dimensions by concatenating vector fields of data with zeros in an extra-dimensional space, and show that the resulting augmented neural ODEs (ANODEs) achieve lower loss and better generalization on image classification and toy experiments.  \citet{li2019deep} show that, if the output of the ODE-Net is composed with elements of a terminal family, then universal approximation capabilities are obtained for the convergence in $L^p$ norm for $p<+ \infty$, which is insufficient \citep{teshima2020universal}. In this work, we consider the representation capabilities in $L^{\infty}$ norm of the ODEs derived from the forward iterations of a ResNet.  Furthermore, \citet{zhang2019approximation} proved that doubling the dimension of the ODE leads to universal approximators, although this result has no application in deep learning to our knowledge.  In this work, we show that in the continuous limit, our architecture has better representation capabilities than Neural ODEs. We also prove its universality in the linear case.
\vspace{-0.5em}
 \paragraph{Momentum in deep networks.} Some recent works \citep{he2020momentum, chun2020momentum,  nguyen2020momentumrnn, li2018optimization} have explored momentum in deep architectures. However, these methods differ from ours in their architecture and purpose. \citet{chun2020momentum} introduce a momentum to solve an optimization problem for which the iterations do not correspond to a ResNet. 
 \citet{nguyen2020momentumrnn} (resp. \citet{he2020momentum}) add momentum in the case of RNNs (different from ResNets) where the weights are tied to alleviate the vanishing gradient issue (resp. link the key and query encoder layers). \citet{li2018optimization} consider a particular case where the linear layer is tied and is a symmetric definite matrix. In particular, none of the mentioned architectures are invertible, which is one of the main assets of our method.
\vspace{-1em}
\paragraph{Second-order models} We show that adding a momentum term corresponds to an Euler integration scheme for integrating a second-order ODE. Some recently proposed architectures \citep{norcliffe2020second, rusch2020coupled, lu2017finite, massaroli2020dissecting} are also motivated by second-order differential equations. \citet{norcliffe2020second} introduce second-order dynamics to model second-order dynamical systems, whereas our model corresponds to a discrete set of equations in the continuous limit. Also, in our method, the neural network only acts on $x$, so that although momentum increases the dimension to $2d$, the computational burden of a forward pass is the same as a ResNet of dimension $d$. \citet{rusch2020coupled} propose second-order RNNs, whereas our method deals with ResNets. Finally, the formulation of LM-ResNet in \citet{lu2017finite} differs from our forward pass ($x_{n+1} = x_n + \gamma  v_n + (1-\gamma)f(x_n,\theta_n)$), even though they both lead to second-order ODEs. Importantly, none of these second-order formulations are invertible.
\vspace{-1em}
\paragraph{Notations}
For $d\in \NN^{*}$, we denote by $\RR^{d \times d}$, $\mathrm{GL}_d{(\RR)}$ and $\mathrm{D}_{d}^{\CC}(\RR)$ the set of real matrices, of invertible matrices, and of \textbf{real} matrices that are diagonalizable in $\CC$. 
\section{Momentum Residual Neural Networks}\label{section:mom_net}
We now introduce Momentum ResNet, a simple transformation of \textbf{any} ResNet into a model with a small memory requirement, and that can be seen in the continuous limit as a second-order ODE.
\vspace{-1em}
\subsection{Momentum ResNets}\label{section:Momentum_nets}

\paragraph{Adding a momentum term in the ResNet equations.}

For \textbf{any} ResNet which iterates~\eqref{eq:ResNet},
we define its Momentum counterpart, which iterates~\eqref{eq:Momentum ResNet},
where $(v_n)_n$ is the velocity  initialized with some value $v_0$ in $\RR^d$, and $\gamma \in [0,1]$ is the so-called momentum term. This approach generalizes gradient descent algorithm with momentum \citep{ruder2016overview},
for which $f$ is the gradient of a function to minimize. 

\paragraph{Initial speed and momentum term.}

 In this paper, we consider initial speeds $v_0$ that depend on $x_0$ through a simple relation. The simplest options are to set $v_0 = 0$ or $v_0 = f(x_0,\theta_0)$. We prove in Section~\ref{section:representation_capabilities} that this dependency between $v_0$ and $x_0$ has an influence on the set of mappings that Momentum ResNets can represent.
The parameter $\gamma$ controls how much a Momentum ResNet diverges from a ResNet, and also the amount of memory saving. The closer $\gamma$ is to $0$, the closer Momentum ResNets are to ResNets, but the less memory is saved. In our experiments, we use $\gamma =0.9$, which we find to work well in various applications.

\paragraph{Invertibility.}

Procedure~\eqref{eq:Momentum ResNet} is inverted through
\begin{equation}\label{eq:backward_momentum}
\left\{
\begin{array}{r@{\hspace{1mm}}l}
x_{n} &=  x_{n+1} -  v_{n+1}, \\
v_{n} &=   \frac1\gamma\left(v_{n+1} - (1-\gamma) f(x_n,\theta_n)\right),\\ 
\end{array}
\right.
\end{equation}
so that activations can be reconstructed on the fly during the backward pass in a Momentum ResNet. 
In practice, in order to exactly reverse the dynamics, the information lost by the finite-precision multiplication by $\gamma$ in~\eqref{eq:Momentum ResNet} has to be efficiently stored. We used the algorithm from \citet{10.5555/3045118.3045343} to perform this reversible multiplication. It consists in maintaining an information buffer, that is, an integer that stores the bits that are lost at each iteration, so that multiplication becomes reversible. We further describe the procedure in Appendix~\ref{app:memory_savings}.  Note that there is always a small loss of floating point precision due to the addition of the learnable mapping $f$. In practice, we never found it to be a problem: this loss in
precision can be neglected compared to the one due to the multiplication by $\gamma$.

\newcommand{\myrot}[1]{\rotatebox{80}{#1}}
 \begin{table}[h!]
    \centering
    \caption{\label{tab:comparison_invertible_architectures} {\bf Comparison of reversible residual architectures} }
    \vskip 0.15in
    \begin{tabular}{ |c|c|c|c|c|c|}
   \cline{2-6}
    \multicolumn{1}{c|}{} & 
  \myrot{Neur.ODE\ } &  \myrot{i-ResNet} & \myrot{i-RevNet} & \myrot{RevNet} & \myrot{\textbf{Mom.Net}}\\
 \hline
   Closed-form inversion & \cmark &  \xmark & \cmark & \cmark & \cmark \\
   \hline
 Same parameters  & \xmark &  \cmark & \xmark & \xmark & \cmark \\
 \hline
 Unconstrained training & \cmark &  \xmark & \cmark & \cmark & \cmark \\ 
 \hline
\end{tabular}
\vskip -0.15in
\end{table}

\paragraph{Drop-in replacement.}
Our approach makes it possible to turn any existing ResNet into a reversible one.  In other words, a ResNet can be transformed into its Momentum counterpart without changing the structure of each layer. For instance, consider a ResNet-152 \citep{he2015deep}. It is made of $4$ layers (of depth $3$, $8$, $36$ and $3$) and can easily be turned into its Momentum ResNet counterpart 
by changing the forward equations~\eqref{eq:ResNet} 
into~\eqref{eq:Momentum ResNet} in the $4$ layers. No further change is needed and Momentum ResNets take the exact same parameters as inputs: they are a \textit{drop-in replacement}. This is not the case of other reversible models. Neural ODEs \citep{chen2018neural} take continuous parameters as inputs. i-ResNets \citep{behrmann2019invertible} cannot be trained by plain SGD since the spectral norm of the weights requires constrained optimization. i-RevNets \citep{jacobsen2018irevnet} and RevNets \citep{gomez2017reversible} require to train two networks with their own parameters for each residual block, split the inputs across convolutional channels, and are half as deep as ResNets: they do not take the same parameters as inputs. Table~\ref{tab:comparison_invertible_architectures} summarizes the properties of reversible residual architectures.  We discuss in further details the differences between RevNets and Momentum ResNets in sections \ref{sec:role_momentum} and \ref{sec-numerics-lista}.

\subsection{Memory cost}\label{section:memory_cost}

Instead of storing the full data at each layer, we only need to store the bits lost at each multiplication by $\gamma$ (cf.\ ``intertibility''). For an architecture of depth $k$, this corresponds to storing $\log_2((\frac{1}{\gamma})^k) $ values for each sample ($\frac{k (1-\gamma)}{\ln(2)}$ if $\gamma$ is close to $1$).
To illustrate, we consider two situations where storing the activations is by far the main memory bottleneck. First, consider a toy feedforward architecture where $f(x,\theta) = W_2^{T} \sigma(W_1x +b)$, with $x \in \RR^d$ and $\theta = (W_1, W_2, b)$, where $W_1, W_2 \in \RR ^{p\times d}$ and $b \in \RR^p$, with a depth $k \in \NN$. We suppose that the weights are the same at each layer.
The training set is composed of $n$ vectors $x^1,...,x^n \in \RR^d$. 
For \textbf{ResNets}, we need to store the weights of the network and the values of all activations for the training set at each layer of the network.
In total, the memory needed is $O(k \times d \times n_{{batch}})$ per iteration. 
In the case of \textbf{Momentum ResNets}, if $\gamma$ is close to $1$ we get a memory requirement of $O((1-\gamma) \times k \times d \times n_{{batch}} )$.
This proves that the memory dependency in the depth $k$ is arbitrarily reduced by changing the momentum $\gamma$. The memory savings are confirmed in practice, as shown in Figure~\ref{fig:memo_theory}. 

\begin{figure}[H]
\centering
\includegraphics[width=0.7\columnwidth]{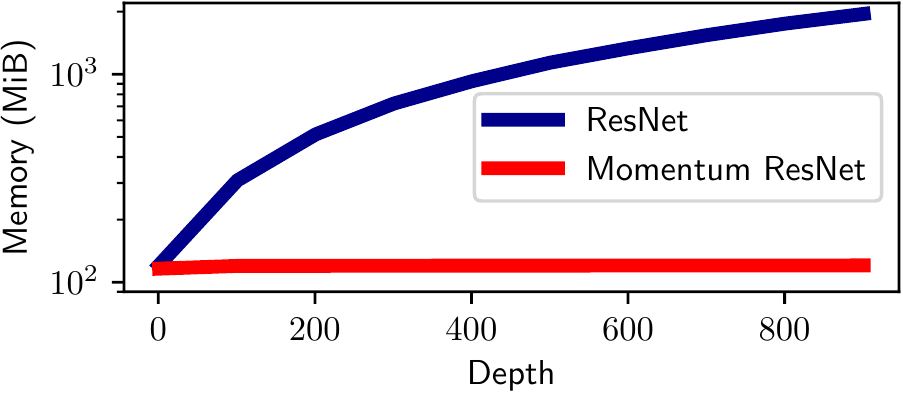} 
\caption{{\bf Comparison of memory needed} (calculated using a profiler) for computing gradients of the loss, with ResNets (activations are stored) and Momentum ResNets (activations are not stored). We set $n_{batch} = 500$, $d=500$ and $\gamma = 1 - \frac{1}{50k}$ at each depth. Momentum ResNets give a nearly constant memory footprint.
}\label{fig:memo_theory}
\vspace{-1em}
\end{figure}

As another example, consider a ResNet-152 \citep{he2015deep} which can be used for ImageNet classification \citep{deng2009imagenet}. Its layer named ``\texttt{conv4\_x}'' has a depth of $36$: it has $40$ M parameters, whereas storing the activations would require storing %
$50$ times more parameters. Since storing the activations is here the main obstruction, the memory requirement for this layer can be arbitrarily reduced by taking $\gamma$ close to $1$.

\subsection{The role of momentum}\label{sec:role_momentum}

When $\gamma$ is set to $0$ in~\eqref{eq:Momentum ResNet}, we recover a ResNet. Therefore, Momentum ResNets are a generalization of ResNets. When $\gamma \xrightarrow[]{} 1$, one can scale $f \to \frac{1}{1 - \gamma} f$ to get in~\eqref{eq:Momentum ResNet}
a symplectic scheme \citep{Hairer:1250576} that recovers a special case of other popular invertible neural network: RevNets \citep{gomez2017reversible} and Hamiltonian Networks \citep{chang2017reversible}. 
A RevNet iterates %
\begin{equation}
     \label{eq:revnet}
    v_{n+1} = v_n + \phi(x_n, \theta_n), 
          \quad
    x_{n+1} = x_n + \psi(v_{n+1}, \theta_n^{'}),
\end{equation}
where $\phi$ and $\psi$ are two learnable functions.

The usefulness of such architecture depends on the task.
RevNets have encountered success for classification and regression.
However, we argue that RevNets cannot work in some settings.
For instance, under mild assumptions, the RevNet iterations do not have attractive fixed points when the parameters are the same at each layer: $\theta_n = \theta$, $\theta_n' = \theta'$.
We rewrite \eqref{eq:revnet} as $(v_{n+1}, x_{n+1}) = \Psi(v_n, x_n)$ with $\Psi(v, x) = (v + \phi(x,\theta), x + \psi(v + \phi(x,\theta),\theta'))$.
\begin{proposition}[Instability of fixed points]
\label{prop:revnet_fix}
Let $(v^*, x^*)$ a fixed point of the RevNet iteration~\eqref{eq:revnet}. Assume that $\phi$ (resp. $\psi$) is differentiable at $x^*$ (resp. $v^*$), with Jacobian matrix $A$ (resp. $B$) $\in\RR^{d\times d}$. The Jacobian of $\Psi$ at $(v^*, x^*)$ is $ J(A, B) =  \big(\begin{smallmatrix}
  \mathrm{Id}_d & A \\
B & \mathrm{Id}_d + BA
\end{smallmatrix}\big)$. 
If $A$ and $B$ are invertible, then there exists $\lambda \in \Sp\left(J(A, B)\right)$ such that $|\lambda| \geq 1$ and $\lambda \neq 1$.
\end{proposition}
This shows that $(v^*, x^*)$ cannot be a stable fixed point.
As a consequence, in practice, a RevNet cannot have converging iterations: according to~\eqref{eq:revnet}, if $x_n$ converges then $v_n$ must also converge, and their limit must be a fixed point. 
The previous proposition shows that it is impossible.

This result suggests that RevNets should perform poorly in problems where one expects the iterations of the network to converge.
For instance, as shown in the experiments in Section~\ref{sec-numerics-lista}, this happens when we use reverible dynamics in order to \emph{learn to optimize} \citep{10.5555/3045118.3045343}.
In contrast, the proposed method can converge to a fixed point as long as the momentum term $\gamma$ is strictly less than $1$.
\vspace{-1em}
\paragraph{Remark.} Proposition \ref{prop:revnet_fix} has a continuous counterpart. Indeed, in the continuous limit, \eqref{eq:revnet} writes $\dot{v} = \phi(x,\theta),  \quad \dot{x} = \psi(v,\theta')$. The corresponding Jacobian in $(v^{*}, x^{*})$ is $\big(\begin{smallmatrix}
  0 & A \\
B &  0
\end{smallmatrix}\big)$. The eigenvalues of this matrix are the square roots of those of $AB$: they cannot all have a real part $< 0$ (same stability issue in the continuous case).

\subsection{Momentum ResNets as continuous models}

\begin{figure}[H]
\includegraphics[width=\columnwidth]{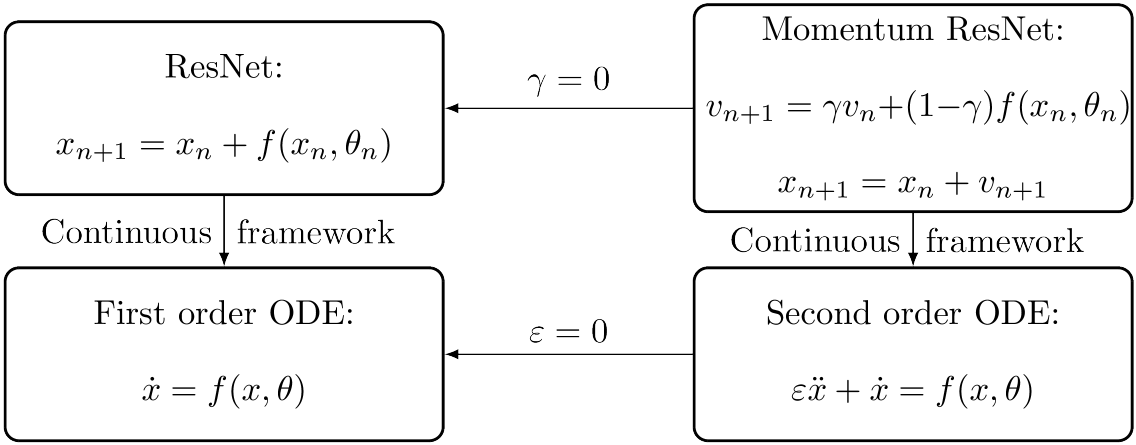} 
\caption{\textbf{Overview of the four different paradigms.} }\label{fig:recap}
\vspace{-1em}
\end{figure}

\paragraph{Neural ODEs: ResNets as first-order ODEs.}
The ResNets equation~\eqref{eq:ResNet} with initial condition $x_0$ (the input of the ResNet) can be seen as a discretized Euler scheme of the ODE $\dot{x}=f(x,\theta)$ with $x(0) = x_0$.
Denoting $T$ a time horizon, the neural ODE maps the input $x(0)$ to the output $x(T)$, and, as in \citet{chen2018neural}, is trained by minimizing a loss $L(x(T),\theta)$. 

\paragraph{Momentum ResNets as second-order ODEs.}

Let $\epsilon = \frac{1}{1-\gamma}$. We can then rewrite \eqref{eq:Momentum ResNet} as 
\begin{equation*}\label{eq:mom_eq}
v_{n+1} = v_{n} +  \frac{f(x_{n},\theta_n)-v_n}{\varepsilon}, 
\quad
x_{n+1} = x_{n} +   v_{n+1},
\end{equation*}
which corresponds to a Verlet integration scheme \citep{Hairer:1250576} with step size $1$ of the differential equation $\varepsilon \ddot{x} + \dot{x} = f(x,\theta). $
Thus, in the same way that ResNets can be seen as discretization of first-order ODEs, Momentum ResNets can be seen as discretization of second-order ones. 
Figure~\ref{fig:recap} sums up these ideas.
\section{Representation capabilities}\label{section:representation_capabilities}
We now turn to the analysis of the representation capabilities of Momentum ResNets in the continuous setting. In particular, we precisely characterize the set of mappings representable by Momentum ResNets with linear residual functions. 
\subsection{Representation capabilities of first-order ODEs}\label{section:representation_capabilitie}
We consider the first-order model 
\begin{equation}\label{eq:first_order_ODE}
  \dot{x} = f(x,\theta)
   \quad\text{with}\quad
  x(0) = x_0. 
\end{equation}

We denote by $\phi_t(x_0)$ the solution at time $t$ starting at initial condition $x(0) = x_0$. It is called the \textit{flow} of the ODE. For all $t \in [0,T]$, where $T$ is a time horizon, $\phi_t$ is a homeomorphism: it is continuous, bijective with continuous inverse. 

\paragraph{First-order ODEs are not universal approximators.}

ODEs such as~\eqref{eq:first_order_ODE} are not universal approximators. Indeed, the function mapping an initial condition to the flow at a certain time horizon $T$ cannot represent every mapping $x_0 \mapsto h(x_0)$. For instance when $d= 1$, the mapping $x \to -x$ cannot be approximated by a first-order ODE, since $1$ should be mapped to $-1$ and $0$ to $0$, which is impossible without intersecting trajectories~\citep{dupont2019augmented}. In fact, the homeomorphisms represented by~\eqref{eq:first_order_ODE} are orientation-preserving: if $K \subset \RR^d$ is a compact set and $h : K\xrightarrow{} \RR^d$ is a homeomorphism represented by~\eqref{eq:first_order_ODE}, then $h$ is in the connected component of the identity function on $K$ for the topology of the uniform convergence (see details in Appendix~\ref{app:prop_connected}). 

\subsection{Representation capabilities of second-order ODEs}
We consider the second-order model for which we recall that Momentum ResNets are a discretization: 
\begin{equation}\label{eq:pertu}
\varepsilon \ddot{x} + \dot{x} = f(x,\theta)
\quad\text{with}\quad
(x(0),\dot{x}(0)) = (x_0,v_0).
\end{equation}
In Section~\ref{sec:role_momentum}, we showed that Momentum ResNets generalize existing models when setting $\gamma = 0$ or $1$. We now state the continuous counterparts of these results.
Recall that $\frac{1}{1-\gamma} = \varepsilon$. When $\varepsilon \xrightarrow{} 0$, we recover the first-order model.
\begin{proposition}[Continuity of the solutions]\label{prop:eps_0}%
We let $x^*$ (resp. $x_{\varepsilon}$) be the solution of~\eqref{eq:first_order_ODE} (resp.~\eqref{eq:pertu}) on $[0,T]$, with initial conditions $x^*(0) = x_{\varepsilon}(0)=x_0$ and $\dot{x}_{\varepsilon}(0) = v_0$.
Then  $\|x_{\varepsilon} - x^* \|_{\infty} \xrightarrow[]{} 0$ as $\varepsilon \xrightarrow{} 0$.
\end{proposition} 
The proof of this result relies on the implicit function theorem and can be found in Appendix~\ref{app:prop_eps_0}.
Note that Proposition~\ref{prop:eps_0} is true whatever the initial speed $v_0$. 
When $\varepsilon \xrightarrow{} +\infty$, one needs to rescale $f$ to study the asymptotics: the solution of $\ddot{x} + \frac{1}{\varepsilon} \dot{x} = f(x,\theta)$ converges to the solution of $\ddot{x} = f(x,\theta)$ (see details in Appendix~\ref{app:prop_eps_infty}).
 These results show that in the continuous regime, Momentum ResNets also interpolate between $\dot{x} = f(x,\theta)$ and $\ddot{x} = f(x,\theta)$.%
 
\paragraph{Representation capabilities of a model~\eqref{eq:pertu} on the $x$ space.}
We recall that we consider initial speeds $v_0$ that can depend on the input $x_0 \in \RR^d$ (for instance $v_0 = 0$ or $v_0 = f(x_0,\theta_0)$). We therefore assume $\phi_t :\RR^d \mapsto  \RR^d$ such that $\phi_t(x_0)$ is solution of~\eqref{eq:pertu}. We emphasize that $\phi_t$ is not always a homeomorphism. For instance, 
$\phi_t(x_0) = x_0\exp{(-t/2)}\cos{(t/2)}$
solves
$   \ddot{x}  + \dot{x} =  - \frac{1}{2} x(t) $ with $
    (x(0),\dot{x}(0)) = (x_0, -\frac{x_0}{2})$. All the trajectories intersect at time $\pi$. It means that Momentum ResNets can learn mappings that are not homeomorphisms, which suggests that increasing $\epsilon$ should lead to better representation capabilities. The first natural question is thus whether, given $h : \RR^d \xrightarrow{} \RR^d$, there exists some $f$ such that $\phi_t$ associated to~\eqref{eq:pertu} satisfies $
\forall x \in \RR^d, \phi_1(x) = h(x)$.
In the case where $v_0$ is an arbitrary function of $x_0$, the answer is trivial since~\eqref{eq:pertu} can represent any mapping, as proved in Appendix~\ref{app:learn_init}.
This setting does not correspond to the common use case of ResNets, which take advantage of their depth, so it is important to impose stronger constraints on the dependency between $v_0$ and $x_0$.  
For instance, the next proposition shows that even if one imposes $v_0 = f(x_0,\theta_0)$, a second-order model is at least as general as a first-order one.

\begin{proposition}[Momentum ResNets are at least as general]\label{prop:bigger_set}
There exists a function $\hat{f}$ such that for all $x$ solution of~\eqref{eq:first_order_ODE}, $x$ is also solution of the second-order model
$\varepsilon \ddot{x} + \dot{x} = \hat{f}(x,\theta)$
with 
$(x(0),\dot{x}(0)) = (x_0,f(x_0,\theta_0))$.
\end{proposition}

Furthermore, even with the restrictive initial condition $v_0=0$, $x \mapsto \lambda x$ for $\lambda > -1$ can always be represented by a second-order model~\eqref{eq:pertu}
(see details in Appendix~\ref{app:prop_lambda}).
This supports the claim that the set of representable mappings increases with $\varepsilon$. 

\subsection{Universality of Momentum ResNets with linear residual functions}

As a first step towards a theoretical analysis of the universal representation capabilities of Momentum ResNets, we now investigate the linear residual function case.
 Consider the second-order linear ODE 
\begin{equation}\label{eq:second_order_lin}
  \varepsilon \ddot{x}  + \dot{x} = \theta x 
  \quad\text{with}\quad
 (x(0),\dot{x}(0)) = (x_0,0),
\end{equation}
with $\theta \in \RR^{d\times d}$. We assume without loss of generality that the time horizon is $T = 1$. We have the following result.
\begin{proposition}[Solution of~\eqref{eq:second_order_lin}]\label{prop:sol_second_order}
At time $1$, \eqref{eq:second_order_lin} defines the linear mapping $x_0 \mapsto \phi_1(x_0) = \Psi_{\varepsilon}(\theta)x_0$ where 
\begin{equation*}
\Psi_{\varepsilon}(\theta)=  e^{-\frac{1}{2\varepsilon}} \sum_{n=0}^{+\infty} \left(\frac{1}{(2n)!} + \frac{1}{2\varepsilon(2n+1)!}\right) \left(\frac{\theta}{\varepsilon} + \frac{\mathrm{Id}_d}{4 \varepsilon^2}\right)^n.
\end{equation*}
\end{proposition}
Characterizing the set of mappings representable by~\eqref{eq:second_order_lin} is thus equivalent to precisely analyzing the range $\Psi_{\varepsilon}(\RR^{d\times d})$.

\paragraph{Representable mappings of a first-order linear model.}

When $\varepsilon \xrightarrow[]{} 0$, Proposition~\ref{prop:eps_0} shows that
 $\Psi_{\varepsilon}(\theta) \xrightarrow{} \Psi_0(\theta) =  \exp \theta$. 
The range of the matrix exponential is indeed the set of representable mappings of a first order linear model 
\begin{equation}\label{eq:first_order_lin}
 \dot{x} = \theta x
 \quad\text{with}\quad
 x(0) = x_0
\end{equation}
and this range is known~\citep{andrica2010image} to be
$
\Psi_0(\RR^{d \times d}) = \exp{(\RR^{d \times d})} = \{M^2 \mid M \in \mathrm{GL}_d(\RR)\}
$.
This means that one can only learn mappings that are the square of invertible mappings with a first-order linear model~\eqref{eq:first_order_lin}. 
To ease the exposition and exemplify the impact of increasing $\epsilon>0$, we now consider the case of matrices with real coefficients that are diagonalizable in $\CC$, $\mathrm{D}_{d}^{\CC}(\RR)$. 
Note that the general setting of arbitrary matrices is exposed in Appendix~\ref{app:th_representability} using Jordan decomposition. Note also that $\mathrm{D}_{d}^{\CC}(\RR)$ is dense in $\RR^{d \times d}$ \citep{10.2307/2160975}. 
Using Theorem~1 from~\citet{culver1966existence}, we have that if $D \in \mathrm{D}_{d}^{\CC}(\RR)$, then $D$ is represented by a first-order model~\eqref{eq:first_order_lin} \textbf{if and only if} $D$ is non-singular and for all eigenvalues $\lambda \in \mathrm{Sp}(D)$ with $\lambda <0$, $\lambda$ is of even multiplicity order. This is restrictive because it forces negative eigenvalues to be in pairs. We now generalize this result and show that increasing $\varepsilon > 0$ leads to less restrictive conditions.
\begin{figure}[H]
\centering
 \includegraphics[width=\columnwidth]{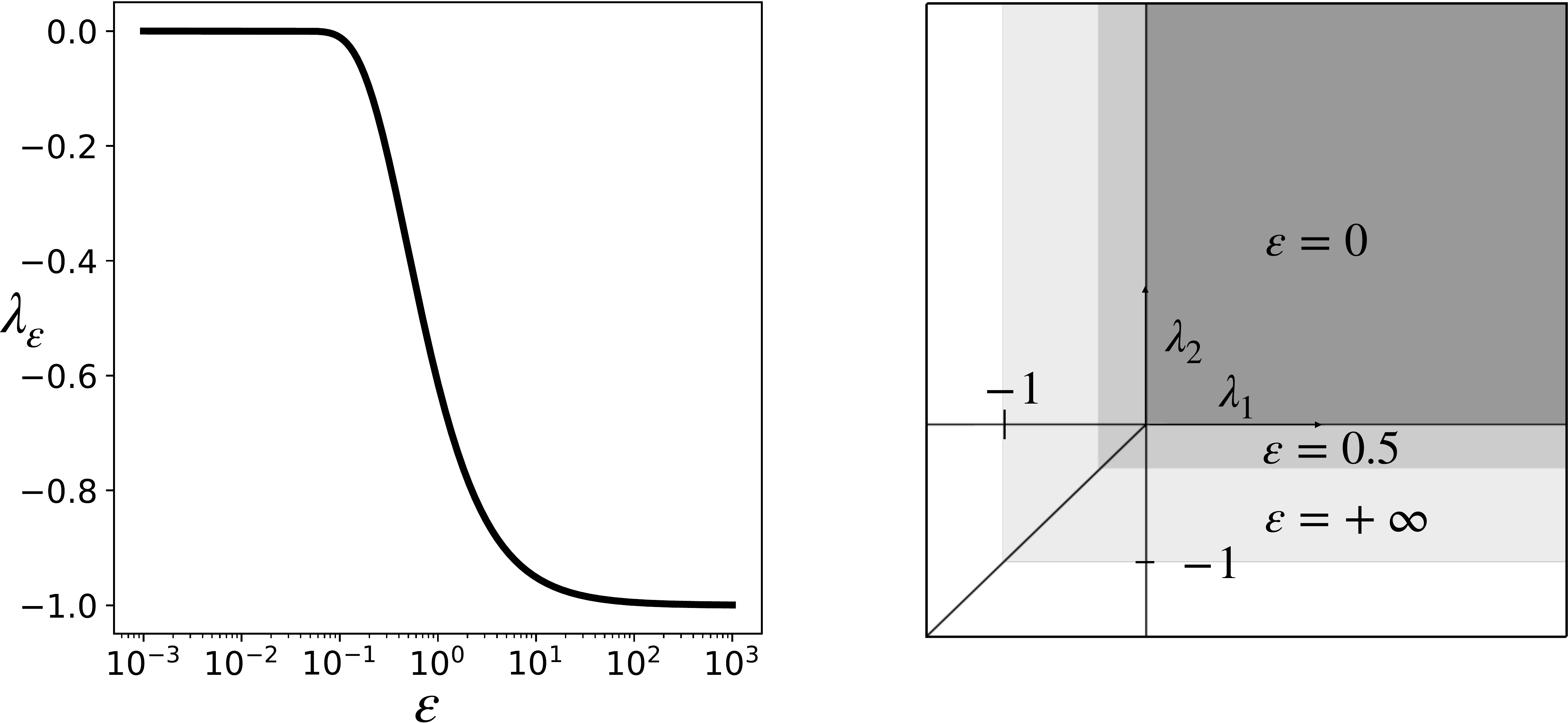}
\caption{Left: \textbf{Evolution of $\lambda_\varepsilon$ defined in Theorem~\ref{th:representability}}. $\lambda_\varepsilon$ is non increasing, stays close to $0$ when $\varepsilon \ll 1$ and close to $-1$ when $\varepsilon \geq 2$. Right: \textbf{Evolution of the real eigenvalues} $\lambda_1$ and $\lambda_2$ of representable matrices in $\mathrm{D}_{d}^{\CC}(\RR)$ by~\eqref{eq:second_order_lin} when $d = 2$ for different values of $\varepsilon$. The grey colored areas correspond to the different representable eigenvalues. When $\varepsilon=0$, $\lambda_1 = \lambda_2$ or $\lambda_1 > 0$ and $\lambda_2 > 0$. When $\varepsilon >0$, single negative eigenvalues are acceptable.}\label{fig:linear}
\vspace{-1em}
\end{figure}
\paragraph{Representable mappings by a second-order linear model.}
Again, by density and for simplicity, we focus on matrices in $\mathrm{D}_{d}^{\CC}(\RR)$, and we state and prove the general case in Appendix~\ref{app:th_representability}, making use of Jordan blocks decomposition of matrix functions \citep{gant} and localization of zeros of entire functions \citep{runckel1969zeros}. 
The range of $\Psi_{\varepsilon}$ over the reals has for form $\Psi_{\varepsilon}(\RR) = [\lambda_\epsilon,+\infty[$. It plays a pivotal role to control the set of representable mappings, as stated in the theorem bellow. Its minimum value can be computed conveniently since it satisfies
$\lambda_\epsilon = \min_{\alpha \in \RR}G_{\varepsilon}(\alpha) $ where
$G_{\varepsilon}(\alpha) \triangleq \exp{(-\frac{1}{2 \varepsilon}})(\cos(\alpha) + \frac{1}{2\varepsilon \alpha} \sin(\alpha))$.
\begin{theorem}[Representable mappings with linear residual functions]\label{th:representability}
Let $D \in \mathrm{D}_{d}^{\CC}(\RR)$.
Then $D$ is represented by a second-order model~\eqref{eq:second_order_lin} \textbf{if and only if} $\forall \lambda \in \mathrm{Sp}(D)$ such that
$\lambda < \lambda_\epsilon $, $\lambda$ is of even multiplicity order.
\end{theorem}
Theorem~\ref{th:representability} is illustrated in Figure \ref{fig:linear}. 
A consequence of this result is that the set of representable linear mappings is \textbf{strictly increasing} with $\epsilon$.
Another consequence is that one can learn \textbf{any} mapping up to scale using the ODE~\eqref{eq:second_order_lin}: if $D \in \mathrm{D}_{d}^{\CC}(\RR)$, there exists $\alpha_{\varepsilon} > 0$ such that for all $\lambda \in \mathrm{Sp}(\alpha_{\varepsilon}D)$, one has $\lambda > \lambda_{\varepsilon}$. Theorem~\ref{th:representability} shows that $\alpha_{\varepsilon}D$ is represented by a second-order model~\eqref{eq:second_order_lin}.

\section{Experiments}

We now demonstrate the applicability of Momentum ResNets through experiments. We used Pytorch and Nvidia Tesla V100 GPUs. 
\subsection{Point clouds separation}\label{sec:point_clouds}
\begin{figure}[H]
\includegraphics[width=\columnwidth]{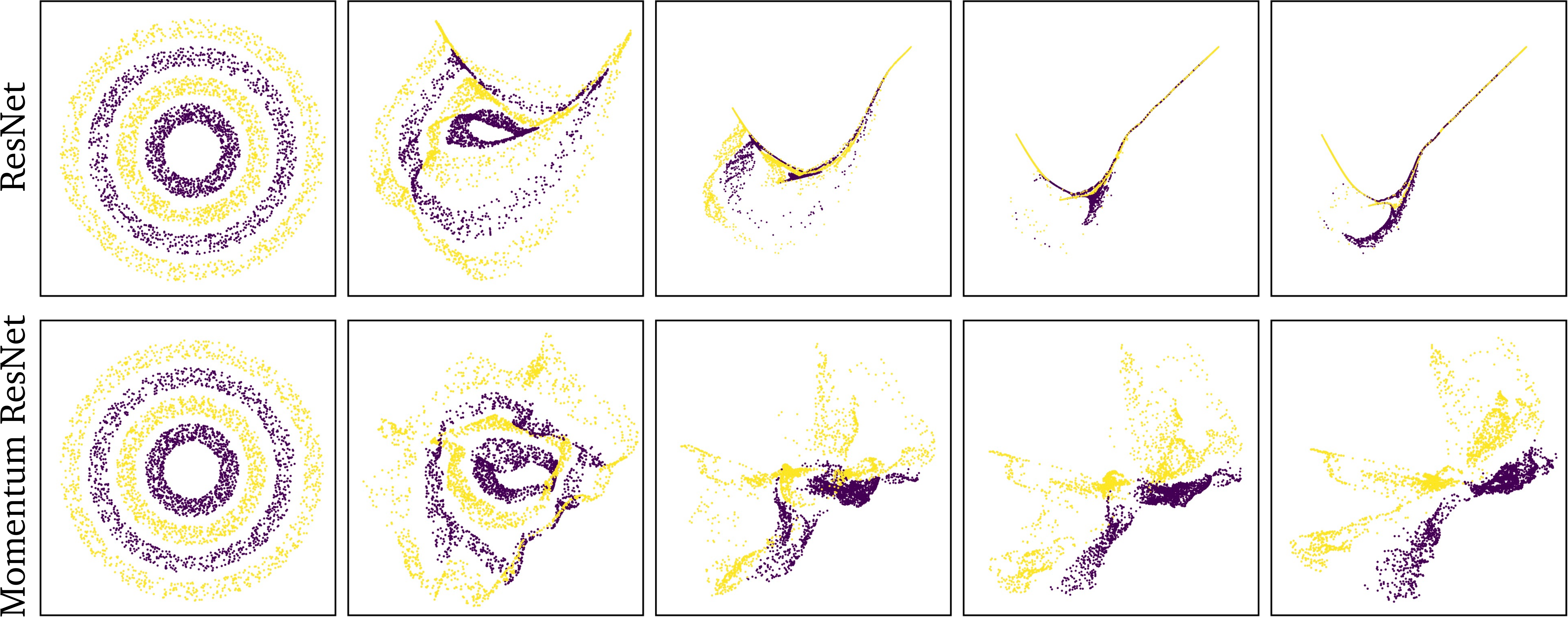} 
\caption{\textbf{Separation of four nested rings} using a ResNet (upper row) and a Momentum ResNet (lower row). From left to right, each figure represents the point clouds transformed at layer $3k$. The ResNet fails whereas the Momentum ResNet succeeds.}\label{fig:four_imbricated}
\vspace{-1em}
\end{figure}
We experimentally validate the representation capabilities of Momentum ResNets on a challenging synthetic classification task. 
As already noted \citep{dupont2019augmented}, neural ODEs ultimately fail to break apart nested rings.
We experimentally demonstrate the advantage of Momentum ResNets by separating $4$ nested rings ($2$ classes). We used the same structure for both models: $f(x,\theta) = W_2^{T} \tanh (W_1x +b)$ with $W_1$, $W_2\in \RR^{16 \times 2}$, $b\in \RR^{16}$, and a depth $15$. Evolution of the points as depth increases is shown in Figure~\ref{fig:four_imbricated}.  The fact that the trajectories corresponding to the ResNet panel don't cross is because, with this depth, the iterations approximate the solution of a first order ODE, for which trajectories cannot cross, due to the Picard-Lindelof theorem.

\subsection{Image experiments}\label{sec:experiments_images}

We also compare the accuracy of ResNets and Momentum ResNets on real data sets: CIFAR-10, CIFAR-100 \citep{CIFAR} and  ImageNet \citep{deng2009imagenet}.  We used existing ResNets architectures. We recall that Momentum ResNets can be used as a drop-in replacement and that it is sufficient to replace every residual building block with a momentum residual forward iteration. We set $\gamma = 0.9$ in the experiments. More details about the experimental setup are given in Appendix~\ref{app:experiment_details}.

\paragraph{Results on CIFAR-10 and CIFAR-100.} 

\begin{table}[H]
\vskip -0.15in
\centering
\caption{\label{tab:results_CIFAR}\textbf{Test accuracy for CIFAR} over 10 runs for each model}
\vskip 0.15in
\begin{adjustbox}{width=\columnwidth,center}
\begin{tabular}{|l|l|l|}
  \hline
  \textbf{Model} & \textbf{CIFAR-10} & \textbf{CIFAR-100} \\ \hline
  {Momentum ResNet, $v_0 = 0$} & {$95.1 \pm 0.13$} & {$76.39 \pm 0.18$} \\ \hline
  {Momentum ResNet, $v_0 = f(x_0)$} & $95.18 \pm 0.06$ & $76.38 \pm 0.42 $ \\ \hline
  {ResNet} & $95.15 \pm 0.12$ & $76.86 \pm 0.25$ \\ \hline
\end{tabular}
\end{adjustbox}

\end{table}

For these data sets,
we used a ResNet-101 \citep{he2015deep} and a Momentum ResNet-101 and compared the evolution of the test error and test loss. Two kinds of Momentum ResNets were used: one with an initial speed $v_0 = 0$ and the other one where the initial speed $v_0$ was learned: $v_0 = f(x_0)$. These experiments show that Momentum ResNets perform similarly to ResNets. Results are summarized in Table~\ref{tab:results_CIFAR}.

\begin{figure*}[ht]
 \centering
\includegraphics[width=0.61\textwidth]{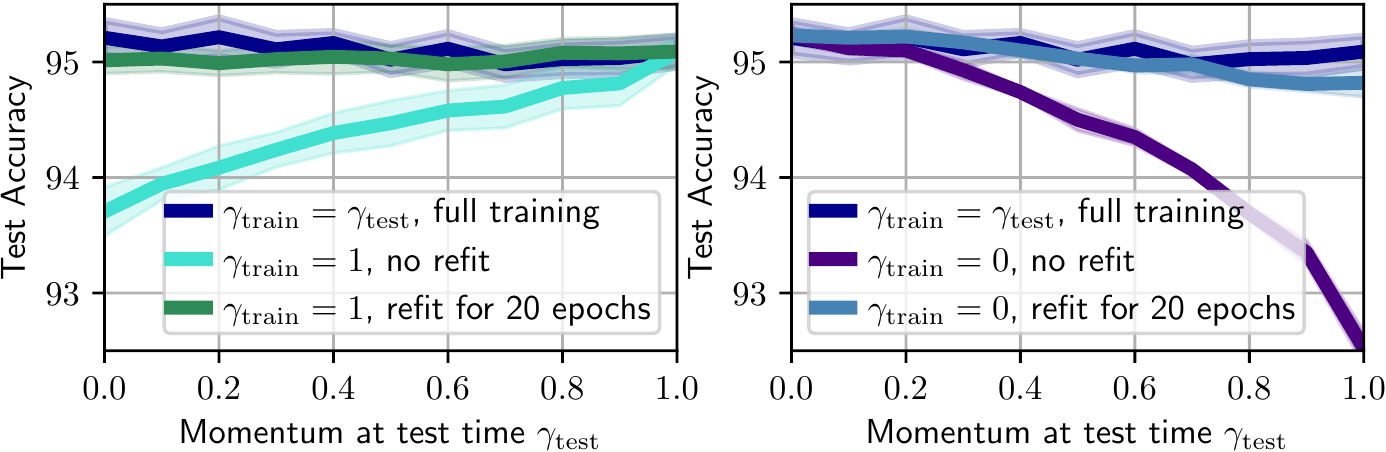} 
 \unskip\ \vrule\
 \includegraphics[width=0.37\textwidth]{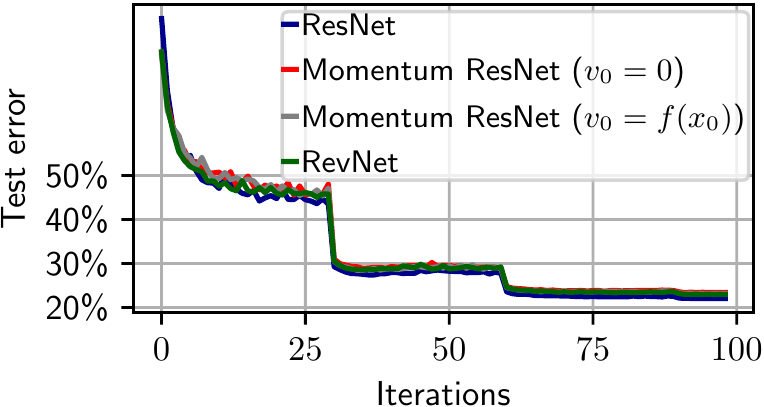} 
 \caption{Upper row: \textbf{Robustness of final accuracy w.r.t $\gamma$} when training Momentum ResNets 101 on CIFAR-10. We train the networks with a momentum $\gamma_{\mathrm{train}}$ and evaluate their accuracy with a different momentum $\gamma_\mathrm{test}$ at test time. We optionally refit the networks for $20$ epochs. We recall that $\gamma_{\mathrm{train}} = 0$ corresponds to a classical ResNet and $\gamma_{\mathrm{train}} = 1$ corresponds to a Momentum ResNet with optimal memory savings.
 Lower row: \textbf{Top-1 classification error on ImageNet (single crop)} for $4$ different residual architectures of depth 101 with the same number of parameters. %
 Final test accuracy is $22 \%$ for the ResNet-101 and $23 \%$ for the $3$ other invertible models.  In particular, our model achieve the same performance as a RevNet with the same number of parameters.}
\label{fig:beta_learning_curves}
\vspace{-1em}
 \end{figure*}
 
\paragraph{Effect of the momentum term $\gamma$.}

Theorem~\ref{th:representability} shows the effect of $\epsilon$ on the representable mappings for linear ODEs.
To experimentally validate the impact of $\gamma$, we train a Momentum ResNet-101 on CIFAR-10 for different values of the momentum at train time, $\gamma_{\mathrm{train}}$. We also evaluate Momentum ResNets trained with $\gamma_{\mathrm{train}} = 0$ and $\gamma_{\mathrm{train}} =1$ with no further training for several values of the momentum at test time, $\gamma_{\mathrm{test}}$. In this case, the test accuracy never decreases by more than $3 \%$.  We also refit for $20$ epochs Momentum ResNets trained with $\gamma_{\mathrm{train}} = 0$ and $\gamma_{\mathrm{train}} =1$. This is sufficient to obtain similar accuracy as models trained from scratch. Results are shown in Figure~\ref{fig:beta_learning_curves} (upper row). This indicates that the choice of $\gamma$ has a limited impact on accuracy.  In addition, learning the parameter $\gamma$ does not affect the accuracy of the model. Since it also breaks the method described in \ref{section:memory_cost}, we fix $\gamma$ in all the experiments.

\paragraph{Results on ImageNet.}
For this data set, we used a ResNet-101, a Momentum ResNet-101, and a RevNet-101. For the latter, we used the procedure from \citet{gomez2017reversible} and adjusted the depth of each layer for the model to have approximately the same number of parameters as the original ResNet-101. Evolution of test errors are shown in Figure~\ref{fig:beta_learning_curves} (lower row), where comparable performances are achieved.
\vspace{-0.5em}
\paragraph{Memory costs.}
We compare the memory (using a memory profiler) for performing one epoch as a function of the batch size for two datasets: ImageNet (depth of 152) and CIFAR-10 (depth of 1201). Results are shown in Figure~\ref{fig:memory_time} and illustrate how Momentum ResNets can benefit from increased batch size, especially for very deep models. We also show in Figure~\ref{fig:memory_time} the final test accuracy for a full training of Momentum ResNets on CIFAR-10 as a function of the memory used (directly linked to $\gamma$ (section \ref{section:memory_cost})).

 \begin{figure}[ht]
 \center  \includegraphics[width=1\columnwidth]{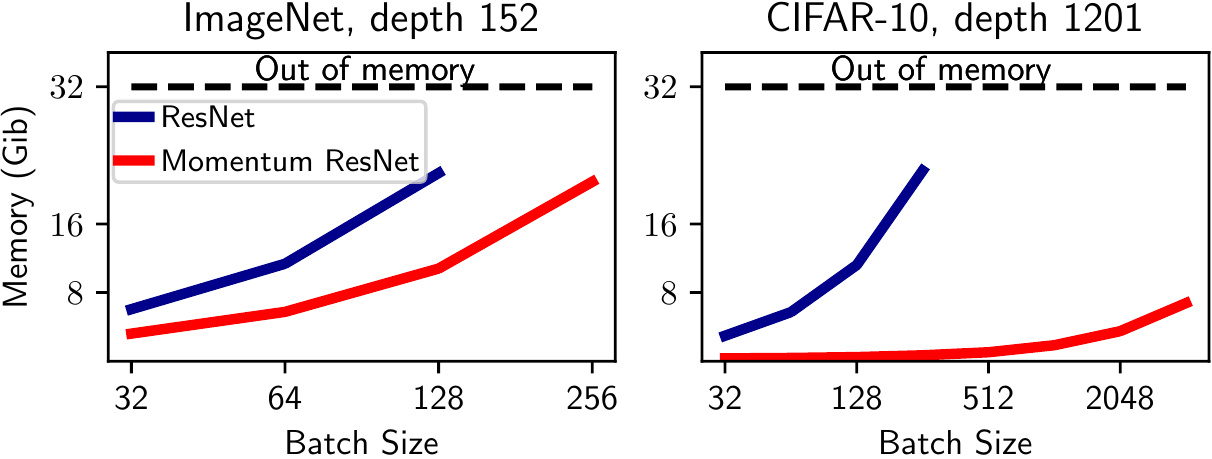}
 \vspace{0.5em}
  \unskip\ \hrule\
  \includegraphics[width=0.63\columnwidth]{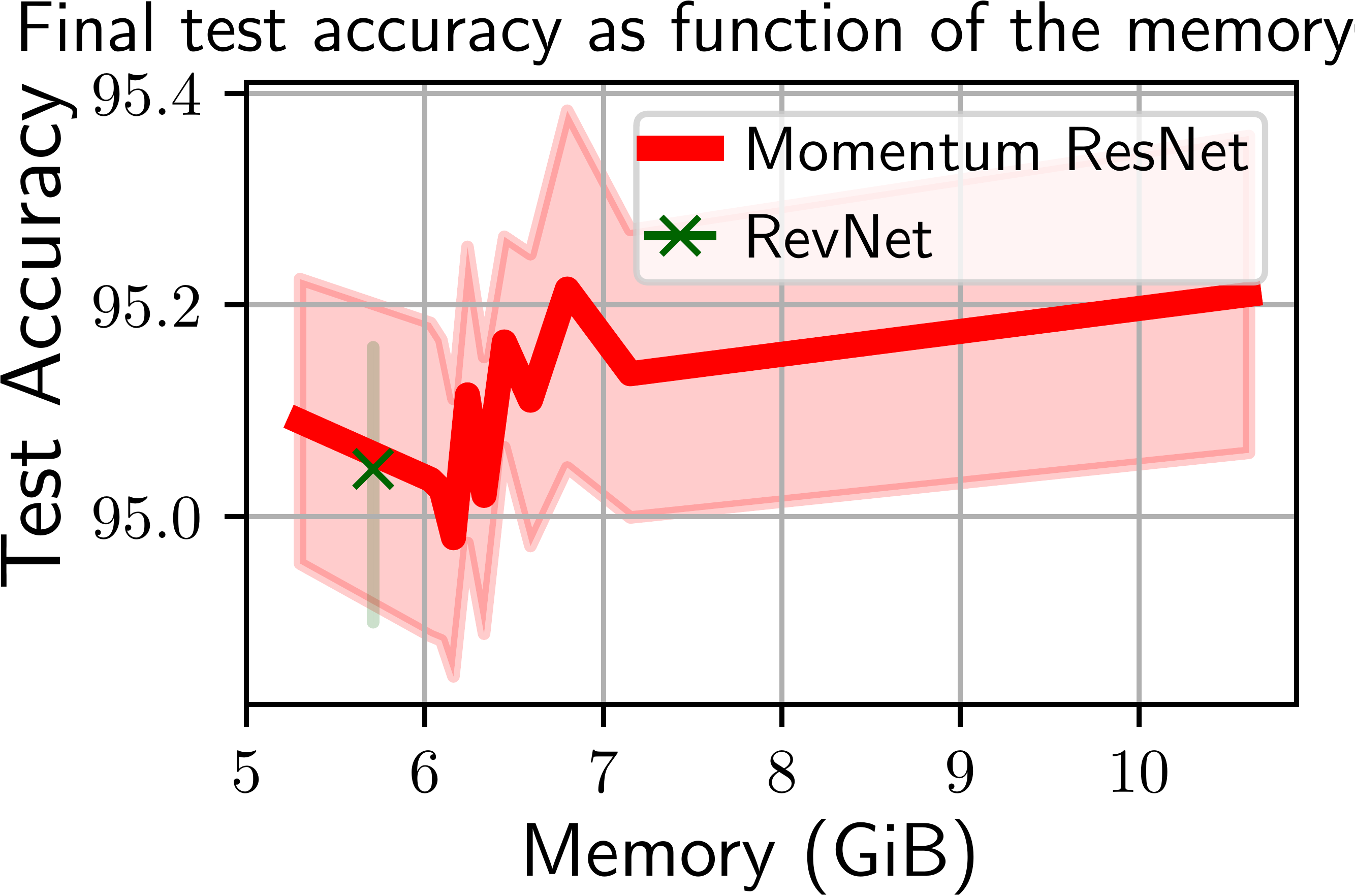}  \caption{Upper row: \textbf{Memory used} (using a profiler) for a ResNet and a Momentum ResNet on one training epoch, as a function of the batch size.  Lower row: \textbf{Final test accuracy} as a function of the memory used (per epoch) for training Momentum ResNets-101 on CIFAR-10.} 
 \label{fig:memory_time} 
 \end{figure}
\paragraph{Ability to perform pre-training and fine-tuning.}

It has been shown \citep{tajbakhsh2016convolutional} that in various medical imaging applications the use of a pre-trained model on ImageNet
with adapted fine-tuning outperformed a model trained from scratch. In order to easily obtain pre-trained Momentum ResNets for applications where memory could be a bottleneck, we transferred the learned parameters of a ResNet-152 pre-trained on ImageNet to a Momentum ResNet-152 with $\gamma=0.9$. In only $1$ epoch of additional training we reached a top-1 error of $26.5 \%$ and in $5$ additional epochs a top-1 error of $23.5 \%$. We then empirically compared the accuracy of these pre-trained models by fine-tuning them on new images: the \texttt{hymenoptera}\footnote{\href{url}{https://www.kaggle.com/ajayrana/hymenoptera-data}} data set. 

\begin{figure}[H]
    \begin{minipage}[c]{0.55\linewidth}
        \includegraphics[width=\textwidth]{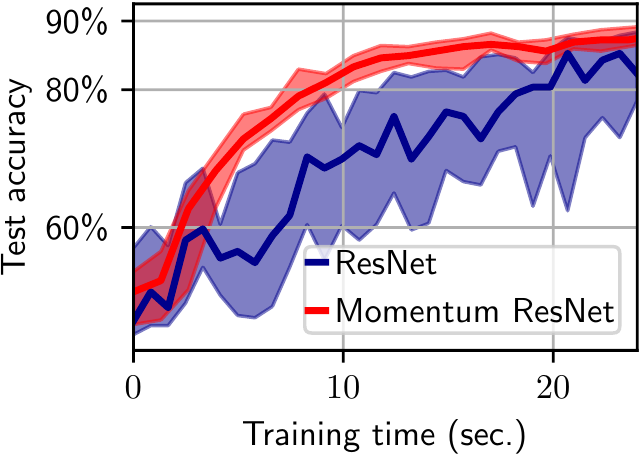}
    \end{minipage}\hfill
    \begin{minipage}[c]{0.42\linewidth}
        \vspace{-1em}
    \caption{\textbf{Accuracy} as a function of time on \texttt{hymenoptera} when fine-tuning a ResNet-152 and a Momentum ResNet-152 with batch sizes of $2$ and $4$, respectively, as permitted by memory. }\label{fig:fine_tuning}
    \end{minipage}
\end{figure}
 \vspace{-1em}
As a proof of concept, suppose we have a GPU with $3$ Go of RAM. The images have a resolution of $500 \times  500$ pixels so that the maximum batch size that can be taken for fine-tuning the ResNet-152 is $2$, against $4$ for the Momentum ResNet-152.  As suggested in \citet{tajbakhsh2016convolutional} (“if the distance between the source and target applications is significant, one may need to fine-tune the early layers as well”), we fine-tune the whole network in this proof of concept experiment.  In this setting the Momentum ResNet leads to faster convergence when fine-tuning, as shown in Figure~\ref{fig:fine_tuning}: Momentum ResNets can be twice as fast as ResNets to train when samples are so big that only few of them can be processed at a time.  In contrast, RevNets \citep{gomez2017reversible} cannot as easily be used for fine-tuning since, as shown in~\eqref{eq:revnet}, they require to train two distinct networks.

\paragraph{Continuous training.}

We also compare accuracy when using first-order ODE blocks \citep{chen2018neural} and second-order ones on CIFAR-10. In order to emphasize the influence of the ODE, we considered a neural architecture which down-sampled the input to have a certain number of channels, and then applied $10$ successive ODE blocks. Two types of blocks were considered: one corresponded to the first-order ODE~\eqref{eq:first_order_ODE} and the other one to the second-order ODE~\eqref{eq:pertu}. Training was based on the odeint function implemented by \citet{chen2018neural}.
Figure~\ref{fig:continuous_CIFAR_10} shows the final test accuracy for both models as a function of the number of channels used. As a baseline, we also include the final accuracy when there are no ODE blocks. We see that an ODE Net with momentum significantly outperforms an original ODE Net when the number of channels is small. Training took the same time for both models. 
\begin{figure}[ht]
    \begin{minipage}[c]{0.6\linewidth}
        \includegraphics[width=\textwidth]{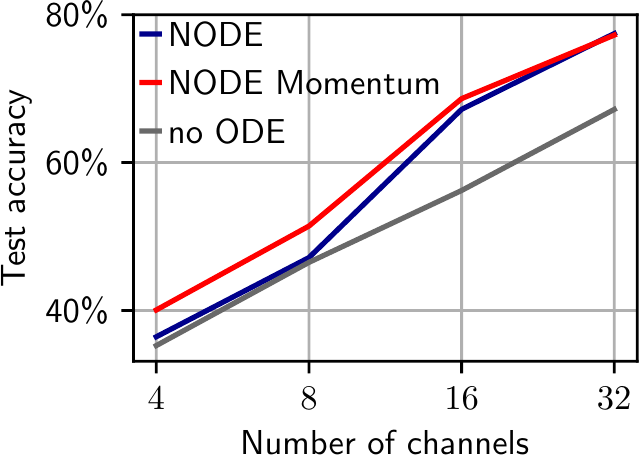}
    \end{minipage}\hfill
    \begin{minipage}[c]{0.35\linewidth}
        \vspace{-1em}
    \caption{\textbf{Accuracy after $120$ iterations on CIFAR-10} with or without momentum, when varying the number of channels.}\label{fig:continuous_CIFAR_10}
    \end{minipage}
    \vspace{-1em}
\end{figure}

\subsection{Learning to optimize}
\label{sec-numerics-lista}

We conclude by illustrating the usefulness of our Momentum ResNets in the \emph{learning to optimize} setting, where one tries to learn to minimize a function.
We consider the Learned-ISTA (LISTA) framework~\citep{gregor2010learning}. Given a matrix $D\in \RR^{d\times p}$, and a hyper-parameter $\lambda>0$, the goal is to perform the sparse coding of a vector $y\in \RR^d$, by finding $x\in\RR^p$ that minimizes the Lasso cost function $\mathcal{L}_y(x) \triangleq \frac12 \|y - Dx\|^2 + \lambda \|x\|_1$~\citep{tibshirani1996regression}.
In other words, we want to compute a mapping $y \mapsto \argmin_x \mathcal{L}_y(x)$.
The ISTA algorithm~\citep{daubechies2004iterative} solves the problem, starting from $x_0 = 0$, by iterating $x_{n+1} = \st(x_n - \eta D^{\top}(Dx_n - y), \eta \lambda)$, with $\eta >0$ a step-size.
Here, $\st$ is the soft-thresholding operator.
The idea of~\citet{gregor2010learning} is to view $L$ iterations of ISTA as the output of a neural network with $L$ layers that iterates $x_{n+1} = g(x_n, y, \theta_n)\triangleq \st(W^1_nx_n + W^2_ny, \eta\lambda)$, with parameters $\theta \triangleq (\theta_1, \dots, \theta_L)$ and $\theta_n \triangleq (W^1_n, W^2_n)$. 
We call $\Phi(y, \theta)$ the network function, which maps $y$ to the output $x_L$.
Importantly, this network can be seen as a residual network, with residual function $f(x, y, \theta) = g(x, y, \theta) - x$.
ISTA corresponds to fixed parameters between layers: $W^1_n = \mathrm{Id}_p - \eta D^{\top}D$ and $W^2_n = \eta D^ {\top}$, but these parameters can be learned to yield better performance.
We focus on an ``unsupervised'' learning setting, where we have some training examples $y^1, \dots, y^Q$, and use them to learn parameters $\theta$ that quickly minimize the Lasso function $\mathcal{L}$. In other words, the parameters $\theta$ are estimated by minimizing the cost function $\theta \mapsto \sum_{q=1}^Q\mathcal{L}_{y_q}(\Phi(y_q, \theta))$.
The performance of the network is then measured by computing the testing loss, that is the Lasso loss on some unseen testing examples.

We consider a Momentum ResNet and a RevNet variant of LISTA which use the residual function $f$. For the RevNet, the activations $x_n$ are first duplicated: the network has twice as many parameters at each layer.
The matrix $D$ is generated with i.i.d. Gaussian entries with  $p=32$, $d=16$, and its columns are then normalized to unit variance.
Training and testing samples $y$ are generated as normalized Gaussian i.i.d. entries. More details on the experimental setup are added in Appendix~\ref{app:experiment_details}. 
The next Figure~\ref{fig:lista} shows the test loss of the different methods, when the depth of the networks varies.
\begin{figure}[H]
\centering
\includegraphics[width=0.8\columnwidth]{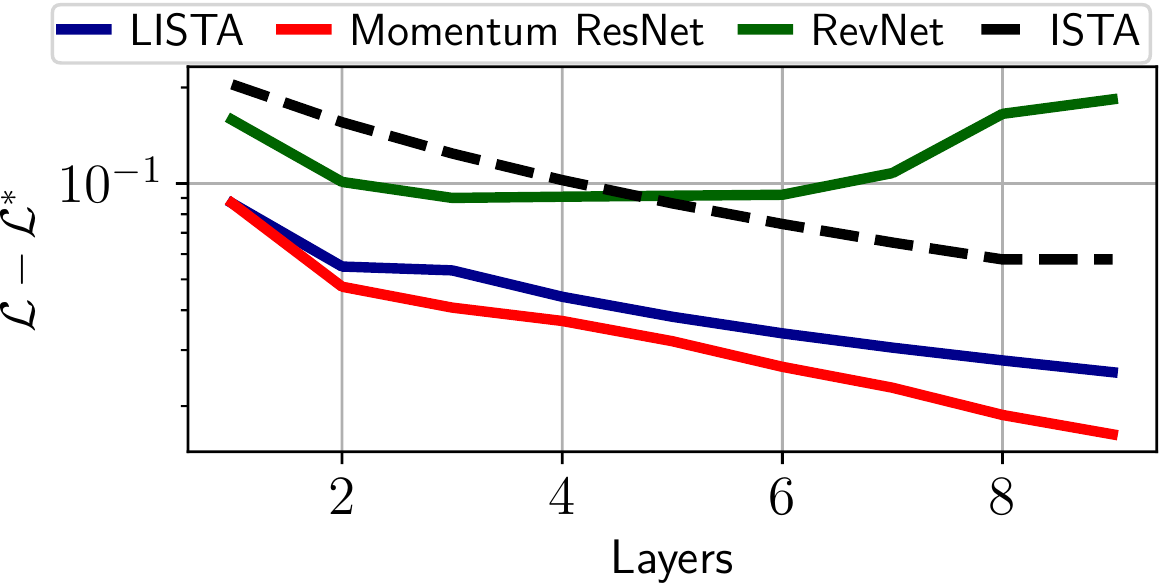} 
\caption{\textbf{Evolution of the test loss} for different models as a function of depth in the Learned-ISTA (LISTA) framework.}
\label{fig:lista}
\end{figure}
As predicted by Proposition~\ref{prop:revnet_fix}, the RevNet architecture fails on this task: it cannot have converging iterations, which is exactly what is expected here.
In contrast, the Momentum ResNet works well, and even outperforms the LISTA baseline.
This is not surprising: it is known that momentum can accelerate convergence of first order optimization methods.
\vspace{-1em}
\section*{Conclusion}

This paper introduces Momentum ResNets, new invertible residual neural networks operating with a significantly reduced memory footprint compared to ResNets. In sharp contrast with existing invertible architectures, they are made possible by a simple modification of the ResNet forward rule. This simplicity offers both theoretical advantages (better representation capabilities, tractable analysis of linear dynamics) and practical ones (drop-in replacement, speed and memory improvements for model fine-tuning). Momentum ResNets interpolate between ResNets ($\gamma=0$) and RevNets ($\gamma=1$), and are a natural second-order extension of neural ODEs. %
As such, they can capture non-homeomorphic dynamics and converging iterations. As shown in this paper, the latter is not possible with existing invertible residual networks, although crucial in the learning to optimize setting.
\vspace{-2em}
\section*{Acknowledgments}

This work was granted access to the HPC resources of IDRIS under the allocation 2020-[AD011012073] made by GENCI. This work was supported in part by the French government under management of Agence Nationale de la Recherche as part of the “Investissements d’avenir” program, reference ANR19-P3IA-0001 (PRAIRIE 3IA Institute). This work was
supported in part by the European Research Council (ERC project NORIA).  The authors would like to thank David Duvenaud and Dougal Maclaurin for their helpful feedbacks.  M. S. thanks Pierre Rizkallah and Pierre Roussillon for fruitful discussions. 
\bibliography{example_paper}
\bibliographystyle{icml2020}

\onecolumn

\icmltitle{Appendix}
\appendix

In Section~\ref{app:proofs} we give the proofs of all the Propositions and the Theorem. In Section~\ref{app:additional_results} we give other theoretical results to validate statements made in the paper. Section~\ref{app:memory_savings} presents the algorithm from \citet{10.5555/3045118.3045343}.  Section~\ref{app:experiment_details} gives details for the experiments in the paper. We derive the formula for backpropagation in Momentum ResNets in Section~\ref{app:backprop_mom_nets}. Finally, we present additional figures in Section~\ref{app:figures}.

\section{Proofs}\label{app:proofs}

\paragraph{Notations}
\begin{itemize}
    \item $C_0^{\infty}([0,1],\RR^d)$ is the set of  infinitely differentiable functions from $[0,1]$ to $\RR^d$ with value $0$ in $0$.
    \item If $f : U \times V \to W$ is a function, we denote by $\partial_u f$, when it exists, the partial derivative of $f$ with respect to $u \in U$.
    \item For a matrix $A \in \RR^{d\times d}$, we denote by $(\lambda - z)^a$ the Jordan block of size $a \in \NN$ associated to the eigenvalue $z \in \CC$ .
\end{itemize}

\setcounter{subsection}{-1}

\subsection{Instability of fixed points -- Proof of Proposition~\ref{prop:revnet_fix}}
\begin{proof}
Since $(x^*, v^*)$ is a fixed point of the RevNet iteration, we have
$$
\phi(x^*) = 0
$$
$$
\psi(v^*)=0
$$

Then, a first order expansion, writing $x = x^*+\varepsilon$ and $v= v^*+\delta$ gives at order one
\begin{align}
    \Psi(v, x) = \left(v^* + \delta + A\varepsilon, x^* + \varepsilon + B(\delta + A\varepsilon)\right)
\end{align}

We therefore obtain at order one
$$
\Psi(v, x) = \Psi(v^*, x^*) + J(A, B)
\begin{pmatrix}
\delta\\
\varepsilon
\end{pmatrix}
$$
which shows that $J(A, B)$ is indeed the Jacobian of $\Psi$ at $(v^*, x^*)$.
We now turn to a study of the spectrum of $J(A, B)$.
We let $\lambda \in \CC$ an eigenvalue of $J(A, B)$, and vectors $u\in\CC^d$, $w\in\CC^d$ such that $(u, w)$ is the corresponding eigenvector, and study the eigenvalue equation
$$
J(A, B)
\begin{pmatrix}
u\\
w
\end{pmatrix}
= \lambda
\begin{pmatrix}
u\\
w
\end{pmatrix}
$$
which gives the two equations
\begin{equation}
    \label{eq:app:eig_eq1}
    u + Aw = \lambda u
\end{equation}
\begin{equation}
\label{eq:app:eig_eq2}    
w + Bu + BAw = \lambda w
\end{equation}

We start by showing that $\lambda \neq 1$ by contradiction.
Indeed, if $\lambda = 1$, then \eqref{eq:app:eig_eq1} gives $Aw = 0$, which implies $w  =0$ since $A$ is invertible. Then, \eqref{eq:app:eig_eq2} gives $Bu=0$, which also implies $u=0$. This contradicts the fact that $(u, v)$ is an eigenvector (which is non-zero by definition).

Then, the first equation~\eqref{eq:app:eig_eq1} gives $Aw = (\lambda - 1)u$, and multiplying \eqref{eq:app:eig_eq2} by $A$ on the left gives
\begin{equation}
    \label{eq:app:eig_eq3}
    \lambda  ABu = (\lambda - 1)^2 u
\end{equation}

We also cannot have $\lambda =0$, since it would imply $u=0$.
Then, dividing \eqref{eq:app:eig_eq3} by $\lambda$ shows that $\frac{(\lambda - 1)^2}{\lambda}$ is an eigenvalue of $AB$.

Next, we let $\mu\neq 0$ the eigenvalue of $AB$ such that $\mu = \frac{(\lambda - 1)^2}{\lambda}$.
The equation can be rewritten as the second order equation 
$$
\lambda^2 -(2 +\mu)\lambda + 1 = 0
$$

This equation has two solutions $\lambda_1(\mu)$, $\lambda_2(\mu)$, and since the constant term is $1$, we have $\lambda_1(\mu)\lambda_2(\mu) = 1$. Taking modulus, we get 
$|\lambda_1(\mu)||\lambda_2(\mu)|= 1$, which shows that necessarily, either $|\lambda_1(\mu)|\geq 1$ or $|\lambda_1(\mu)|\geq 1$.

Now, the previous reasoning is only a necessary condition on the eigenvalues, but we can now prove the advertised result by going backwards: we let $\mu \neq 0$ an eigenvalue of $AB$, and $u\in\CC^d$ the associated eigenvector. We consider $\lambda$ a solution of $\lambda^2 -(2 +\mu)\lambda + 1 = 0$ such that $|\lambda | \geq 1$ and $\lambda \neq 1$. Then, we consider $w = (\lambda - 1)A^{-1}u$. We just have to verify that $(u, v)$ is an eigenvector of $J(A, B)$. By construction, \eqref{eq:app:eig_eq1} holds.
Next, we have 
$$
A(w + Bu + BAw) = (\lambda - 1)u + ABu + (\lambda -1)ABu=(\lambda - 1) u + \lambda ABu
$$
Leveraging the fact that $u$ is an eigenvector of $AB$, we have $\lambda ABu = \lambda \mu u$, and finally:

$$
A(w + Bu + BAw) = (\lambda - 1 + \lambda \mu)u = \lambda (\lambda-1)u = \lambda Aw
$$
Which recovers exactly \eqref{eq:app:eig_eq2}: $\lambda$ is indeed an eigenvalue of $J(A, B)$.
\end{proof}

\subsection{Momentum ResNets in the limit $\varepsilon \xrightarrow[]{} 0$ -- Proof of Proposition~\ref{prop:eps_0}}\label{app:prop_eps_0}

\begin{proof}
We take $T=1$ without loss of generality.
We are going to use the implicit function theorem. 
Note that $x_{\varepsilon}$ is solution of ~\eqref{eq:pertu} if and only if $(x_{\varepsilon}, v_{\varepsilon} = \dot{x_{\varepsilon}})$ is solution of
\begin{equation*}
\begin{split}
\begin{cases}
\dot{x} &= v, \quad x(0) = x_0 \\
\varepsilon \dot{v} &= f(x,\theta) - v, \quad v(0) = v_0.
\end{cases}
\end{split}
\end{equation*}
Consider for $u = (x,v) \in (x_0,v_0) + C_0^{\infty}([0,1],\RR^d)^2$ 
$$\Psi(u,\varepsilon) = \left(x_0 - x +  \int_0^{t}v , \int_0^{t}(f(x,\theta) - v) - \varepsilon v + \varepsilon v_0 \right),$$
so that $x_{\varepsilon}$ is solution of ~\eqref{eq:pertu} if and only if $u_{\varepsilon} = (x_{\varepsilon}, v_{\varepsilon} = \dot{x_{\varepsilon}})$ satisfies $\Psi(u_{\varepsilon},\varepsilon) = 0.$ 
Let $u^{*} = (x^*,\dot{x^*})$.
One has $\Psi(u^{*},0) = 0.$
$\Psi$ is differentiable everywhere, and at $(u^{*},0)$ we have
$$
\partial_{u}\Psi (u^*,0)(x,v) = \left( (\int_0^{t}v) - x, \int_0^{t}(\partial_x f(x*,\theta).x - v)\right).
$$
$\partial_{u}\Psi (u^*,0)$ is continuous, and it is invertible with continuous inverse because it is linear and continuous, and because $\partial_{u}\Psi (u^*,0)(x,v) = 0$ if and only if
$$
\begin{cases}
\forall t \in [0,1], x(t) = \int_0^{t} v \\
\forall t \in [0,1], v(t) =\partial_x f(x^*(t),\theta(t)).x(t)
\end{cases}
$$
which is equivalent to 
\begin{equation*}
\begin{cases}
\dot{x}  = \partial f(x^*,\theta).x\\
x(0) = 0 \\
v  = \dot{x}, \\
\end{cases}
\end{equation*}
which is equivalent, because this equation is linear to $(x,v) = (0,0)$.
Using the implicit function theorem, we know that there exists two neighbourhoods $U \subset \RR$ and $V \subset (x_0,v_0) + C_0^{\infty}([0,1],\RR^d)^2$ of $0$ and $u^*$ and a continuous function $\zeta : U \to V$ such that
$$
\forall (u,\varepsilon) \in U \times V, \Psi(u,\varepsilon) = 0 \Leftrightarrow u = \zeta (\varepsilon)
$$
This in particular ensures that $x_{\epsilon}$ converges uniformly to $x^*$ as $\epsilon$ goes to $0$
\end{proof}

\subsection{Momentum ResNets are more general than neural ODEs -- Proof of Proposition~\ref{prop:bigger_set}}
\begin{proof}
If $x$ satisfies~\eqref{eq:first_order_ODE} we get by derivation that
$$
\ddot{x} = \partial_xf(x,\theta)f(x,\theta) + \partial_\theta f(x,\theta)\dot{\theta}
$$
Then, if we define $\hat{f}(x,\theta) = \varepsilon[\partial_xf(x,\theta)f(x,\theta) + \partial_\theta f(x,\theta)\dot{\theta}] + f(x,\theta)$, we get that $x$ is also solution of the second-order model
$\varepsilon \ddot{x} + \dot{x} = \hat{f}(x,\theta)$
with 
$(x(0),\dot{x}(0)) = (x_0,f(x_0,\theta_0))$.
\end{proof}

\subsection{Solution of~\eqref{eq:second_order_lin} -- Proof of Proposition~\ref{prop:sol_second_order}}

\eqref{eq:second_order_lin} writes

\begin{equation*}
\begin{split}
\begin{cases}
\dot{x} &= v, \quad x(0) = x_0 \\
\dot{v} &= \frac{\theta x - v}{\varepsilon}, \quad v(0) = 0.
\end{cases}
\end{split}
\end{equation*}

For which the solution at time $t$ writes

$$
\begin{pmatrix} 
x(t) \\ v(t) 
\end{pmatrix} = \exp{\begin{pmatrix}
0 & \mathrm{Id}_dt \\
\frac{\theta t}{\varepsilon} & - \frac{ \mathrm{Id}_d t}{\varepsilon} 
\end{pmatrix}} . \begin{pmatrix} 
x_0 \\ 0
\end{pmatrix}.
$$

The calculation of this exponential gives
$$
x(t) =  e^{-\frac{t}{2\varepsilon}} \left(\sum_{n=0}^{+\infty} \frac{1}{(2n)!}(\frac{\theta}{\varepsilon} + \frac{\mathrm{Id}_d}{4 \varepsilon^2})^nt^{2n} + \sum_{n=0}^{+\infty} \frac{1}{2\varepsilon(2n+1)!}(\frac{\theta}{\varepsilon} + \frac{\mathrm{Id}_d}{4 \varepsilon^2})^nt^{2n+1}\right)x_0.
$$
Note that it can be checked directly that this expression satisfies~\eqref{eq:second_order_lin} by derivations.
At time $1$ this effectively gives $x(1) = \Psi_{\varepsilon}(\theta)x_0$.

\subsection{Representable mappings for a Momentum ResNet with linear residual functions -- Proof of Theorem~\ref{th:representability}}\label{app:th_representability}

In what follows, we denote by $f_{\varepsilon}$ the function of matrices defined by $$f_{\varepsilon}(\theta) = \Psi_{\varepsilon}(\varepsilon \theta - \frac{I}{4\varepsilon}) = e^{-\frac{1}{2\varepsilon}} \sum_{n=0}^{+\infty} \left(\frac{1}{(2n)!} + \frac{1}{2\varepsilon(2n+1)!}\right) \theta ^n.$$

Because $\Psi_{\varepsilon}(\RR^{d\times d}) = f_{\varepsilon}(\RR^{d\times d})$, we choose to work on $f_{\varepsilon}$.

We first need to prove that $f_{\varepsilon}$ is surjective on $\CC$.

\subsubsection{Surjectivity on $\CC$ of $f_{\varepsilon}$}
\begin{lemma}[Surjectivity of $f_{\varepsilon}$]\label{lemma:surjectivity}
For $\varepsilon>0$, $f_{\varepsilon}$ is surjective on $\CC$. 
\end{lemma}

\begin{proof}
Consider
\begin{align*}
 F_{\varepsilon} \colon \CC & \longrightarrow \CC\\
z &\longmapsto e^{-\frac{1}{2 \varepsilon}}(\cosh(z) + \frac{1}{2\varepsilon z} \sinh(z)).
\end{align*}
For $z \in \CC$, we have $f_{\varepsilon}(z^2) = F_{\varepsilon}(z)$, and because $z \mapsto z^2$ is surjective on $\CC$, it is sufficient to prove that $F_{\varepsilon}$ is surjective on $\CC$.
Suppose by contradiction that there exists $w \in \CC$ such that $\forall z \in \CC$, $\exp{(\frac{1}{2 \varepsilon}})F_{\varepsilon}(z) \neq w$.
Then $\exp{(\frac{1}{2 \varepsilon}})F_{\varepsilon} -w $ is an entire function  \cite{levin1996lectures} of order 1 with no zeros. Using Hadamard's factorization theorem \cite{conway2012functions}, this implies that there exists $a, b \in \CC$ such that $\forall z \in \CC$, 
$$
\cosh(z) + \frac{\sinh(z)}{2\varepsilon z} - w = \exp{(az +b).}
$$
However, since $F_{\varepsilon}$ is an even function one has that $\forall z \in \CC$ 
$$
\exp{(az +b)} = \exp{(-az +b)}
$$
so that $\forall z \in \CC$, $2az \in 2i\pi\ZZ$. Necessarily, $a = 0$, which is absurd because $F_{\varepsilon}$ is not constant.\\
\end{proof}

We first prove Theorem~\ref{th:representability} in the diagonalizable case. 
\subsubsection{Theorem~\ref{th:representability} in the diagonalizable case}

\begin{proof}

\textbf{Necessity}
Suppose that $D$ can be represented by a second-order model~\eqref{eq:second_order_lin}. This means that there exists a real matrix $X$ such that $D = f_{\varepsilon}(X)$
with $X$ real and
$$f_{\varepsilon}(X)=  e^{-\frac{1}{2\varepsilon}} (\sum_{n=0}^{+\infty} a_n^{\varepsilon} X^n)$$
with 
$$
a_n^{\varepsilon} = \frac{1}{(2n)!} + \frac{1}{2\varepsilon(2n+1)!}.
$$
$X$ commutes with $D$ so that there exists $P \in \mathrm{GL}_d(\CC)$ such that $P^{-1}DP$ is diagonal and $P^{-1}XP$ is triangular. Because $f_{\varepsilon}(P^{-1}XP) = P^{-1}DP$, we have that $\forall \lambda \in \mathrm{Sp}(D)$,  there exists $z \in \mathrm{Sp}(X)$ such that $\lambda = f_{\varepsilon}(z)$. Because $\lambda < \lambda_{\varepsilon}$, necessarily, $z \in \CC - \RR$. In addition, $\lambda = f_{\varepsilon}(z) = \bar{\lambda} = f_{\varepsilon}(\bar{z})$. Because $X$ is real, each $z \in \mathrm{Sp}(X)$ must be associated with $\bar{z}$ in $P^{-1}XP$. Thus, $\lambda$ appears in pairs in $P^{-1}DP$.
\paragraph{Sufficiency}
Now, suppose that $\forall \lambda \in \mathrm{Sp}(D)$ with $\lambda < \lambda_{\varepsilon}$,  $\lambda$ is of even multiplicity order. We are going to exhibit a $X$ real such that $D = f_{\varepsilon}(X)$. Thanks to Lemma~\ref{lemma:surjectivity}, we have that $f_{\varepsilon}$ is surjective. Let $\lambda \in \mathrm{Sp}(D)$.
\begin{itemize}
    \item If $\lambda \in \RR$ and $\lambda < \lambda_{\varepsilon}$ or $\lambda \in \CC - \RR$ then there exists $z \in \CC - \RR$ by Lemma~\ref{lemma:surjectivity} such that $\lambda = f_{\varepsilon}(z)$. 
    \item If $\lambda \in \RR$ and $\lambda \geq \lambda_{\varepsilon}$, then because $f_{\varepsilon}$ is continuous and goes to infinity when $x \in \RR$ goes to infinity, there exists $x \in \RR$ such that $\lambda = f_{\varepsilon}(x)$. 
  
\end{itemize}

In addition, there exist $(\alpha_1, ..., \alpha_k) \in (\CC - \RR)^k \cup [-\infty, \lambda_{\varepsilon}[ ^k$, $(\beta_1,...,\beta_p) \in [\lambda_{\varepsilon},+\infty]^p$ such that 
$$
D = Q^{-1} \Delta Q,
$$
with $Q\in \mathrm{GL}_d{(\RR)}$,
and 
\begin{equation*}
\Delta=
\begin{pmatrix}
P_1^{-1}D_{\alpha_1} P_1 & 0_2 & \cdots & \cdots & \cdots & 0_2 \\
0_2 & \ddots & \cdots & \cdots & \cdots  & 0_2 \\
 \vdots & \vdots & P_k^{-1} D_{\alpha_k} P_k & 0_2 &  \cdots & 0_2 \\
0 & \cdots & \cdots & \beta_1 & \cdots   & 0 \\
0 & \cdots & \cdots & 0 & \ddots   & 0 \\
0 & \cdots & \cdots & \cdots & \cdots & \beta_p
\end{pmatrix} \in \RR^{d \times d}
\end{equation*}
with $P_j \in GL_2(\CC)$ and $D_{\alpha_j} = \begin{pmatrix}
\alpha_j & 0\\
0 & \bar{\alpha_j}
\end{pmatrix}$.
\\

Let $(z_1, ..., z_k) \in (\CC - \RR)^k $ and $(x_1,...,x_p) \in \RR^p$ be such that $f_{\varepsilon}(z_j) = \alpha_j$ and $f_{\varepsilon}(x_j) = \beta_j$. For $1 \leq j \leq k$, one has $P_j^{-1} D_{z_j} P_j \in \RR^{2\times2}$. Indeed, writing $\alpha_j = a_j + i b_j$ with $a_j, b_j \in \RR$, the fact that $P_j^{-1} D_{\alpha_j} P_j \in \RR^{2\times2}$ implies that $i \begin{pmatrix}
1 & 0\\
0 & -1
\end{pmatrix} \in i \RR^{2\times2}$. Writing $z_j = u_j + i v_j $ with $u_j, v_j \in \RR$, we get that $P_j^{-1} D_{z_j} P_j \in \RR^{2\times2}$.
Then
\begin{equation*}
X= Q 
\begin{pmatrix}
P_1^{-1}D_{z_1} P_1 & 0_2 & \cdots & \cdots & \cdots & 0_2 \\
0_2 & \ddots & \cdots & \cdots & \cdots  & 0_2 \\
 \vdots & \vdots & P_k^{-1} D_{z_k} P_k & 0_2 &  \cdots & 0_2 \\
0 & \cdots & \cdots & x_1 & \cdots   & 0 \\
0 & \cdots & \cdots & 0 & \ddots   & 0 \\
0 & \cdots & \cdots & \cdots & \cdots & x_p
\end{pmatrix} Q^{-1} \in \RR^{d \times d}
\end{equation*}
is such that $f_{\varepsilon}(X) = D$, and $D$ is represented by a second-order model~\eqref{eq:second_order_lin}.
\end{proof} 

We now state and demonstrate the general version of Theorem~\ref{th:representability}.

First, we need to demonstrate properties of the complex derivatives of the entire function $f_{\varepsilon}$.

\subsubsection{The entire function $f_{\varepsilon}$ has a derivative with no-zeros on $\CC-\RR$.}

\begin{lemma}[On the zeros of $f_{\varepsilon}'$]\label{lemma:zeros_derivative}
$\forall z \in \CC-\RR$ we have $f_{\varepsilon}'(z) \neq 0$.
\end{lemma}

\begin{proof}
One has $$G_{\varepsilon}(z) =  e^{-\frac{1}{2 \varepsilon}}(\cos(z) + \frac{1}{2\varepsilon z} \sin(z)) = f_{\varepsilon}(-z^2)$$
so that $G_{\varepsilon}'(z) = -2zf_{\varepsilon}'(-z^2)$ and it is sufficient to prove that the zeros of $G_{\varepsilon}'$ are all real. 

We first show that $G_{\varepsilon}$ belongs to the Laguerre-Pólya class \cite{craven2002iterated}. The Laguerre-Pólya class is the set of entire functions that are the uniform limits on compact sets of $\CC$ of polynomials with only real zeros. To show that $G_{\varepsilon}$ belongs to the Laguerre-Pólya class, it is sufficient to show \citep[p. 22]{dryanov1999approximation} that:
\begin{itemize}
    \item The zeros of $G_{\varepsilon}$ are all real.
    \item If $(z_n)_{n \in \NN}$ denotes the sequence of real zeros of $G_{\varepsilon}$, one has $\sum \frac{1}{|z_n|^2} < \infty$.
    \item $G_{\varepsilon}$ is of order $1$.
\end{itemize}
First, the zeros of $G_{\varepsilon}$ are all real, as demonstrated in \citet{runckel1969zeros}. Second, if $(z_n)_{n \in \NN}$ denotes the sequence of real zeros of $G_{\varepsilon}$, one has $z_n \sim n \pi + \frac{\pi}{2}$ as $n\xrightarrow[]{} \infty$, so that $\sum \frac{1}{|z_n|^2} < \infty$. Third, $G_{\varepsilon}$ is of order $1$.
Thus, we have that $G_{\varepsilon}$ is indeed in the Laguerre-Pólya class.

This class being stable under differentiation, we get that $G_{\varepsilon}'$ also belongs to the Laguerre-Pólya class. So that the roots of $G_{\varepsilon}'$ are all real, and hence those of $f_{\varepsilon}$ as well.

\end{proof}
\subsubsection{Theorem~\ref{th:representability} in the general case}

When $\varepsilon = 0$, we have in the general case the following from \citet{culver1966existence}:

Let $A \in \RR^{d \times d}$. Then $A$ can be represented by a first-order model~\eqref{eq:first_order_lin} \textbf{if and only if} $A$ is not singular and each Jordan block of $A$ corresponding to an eigen value  $\lambda < 0$ occurs an even number of time. 

We now state and demonstrate the equivalent of this result for second order models~\eqref{eq:second_order_lin}.

\begin{theorem}[Representable mappings for a Momentum ResNet with linear residual functions -- General case]

Let $A \in \RR^{d \times d}$.

If $A$ can be represented by a second-order model~\eqref{eq:second_order_lin}, then each Jordan block of A corresponding to an eigen value  $\lambda < \lambda_{\varepsilon}$ occurs an even number of time.

Reciprocally, if each Jordan block of A corresponding to an eigen value  $\lambda \leq \lambda_{\varepsilon}$ occurs an even number of time, then $A$ can be represented by a second-order model. 
\end{theorem}
\begin{proof}

We refer to the arguments from \citet{culver1966existence} and use results from \citet{gant} for the proof. 

Suppose that $A$ can be represented by a second-order model~\eqref{eq:second_order_lin}. This means that there exists $X \in \RR^{d\times d}$ such that $\mathit{A = f_{\varepsilon}(X)}$. The fact that $X$ is real implies that its Jordan blocks are:
\begin{equation*}
    \begin{split}
        &(\lambda - z_k)^{a_k}  ,\, z_k \in \RR \\
         &(\lambda - z_k)^{b_k} \text{ and } (\lambda - \bar{z_k})^{b_k} , \, z_k \in \CC -\RR .
    \end{split}
\end{equation*}

Let $\lambda_k = f_{\varepsilon}(z_k)$ be an eigenvalue of $A$ such that $\lambda_k < \lambda_{\varepsilon}$. Necessarily, $z_k \in \CC -\RR$, and $f_{\varepsilon}'(z_k) \neq 0$ thanks to Lemma~\ref{lemma:zeros_derivative}. We then use Theroem 9 from \citet{gant} (p. 158) to get that the Jordan blocks of $A$ corresponding to  $\lambda_k$ are 
\begin{equation*}
         (\lambda - f_{\varepsilon}(z_k))^{b_k} \text{ and } (\lambda - f_{\varepsilon}(\bar{z_k}))^{b_k}.
\end{equation*}
 Since  $f_{\varepsilon}(\bar{z_k}) = f_{\varepsilon}(z_k) = \lambda_k$, we can conclude that the Jordan blocks of A corresponding $\lambda_k < \lambda_{\varepsilon}$ occur an even number of time.

Now, suppose that each Jordan block of $A$ corresponding to an eigen value 
$
\lambda \leq \lambda_{\varepsilon} 
$
occurs an even number of times.
Let $\lambda_k$ be an eigenvalue of $A$.
\begin{itemize}
    \item If $\lambda_k \in \CC-\RR$  we can write, because $f_{\varepsilon}$ is surjective (proved in Lemma~\ref{lemma:surjectivity}), $\lambda_k = f_{\varepsilon}(z_k)$ with $z_k \in \CC-\RR$. Necessarily, because $A$ is real, the Jordan blocks of $A$ corresponding to $\lambda_k$ have to be associated to those corresponding to $\bar{\lambda_k}$. In addition, thanks to Lemma~\ref{lemma:zeros_derivative}, $f_{\varepsilon}'(z_k) \neq 0$ 
    \item If $\lambda_k < \lambda_{\varepsilon}$, we can write, because $f_{\varepsilon}$ is surjective, $\lambda_k = f_{\varepsilon}(z_k) = f_{\varepsilon}(\bar{z_k})$ with $z_k \in \CC-\RR$. In addition, $f_{\varepsilon}'(z_k) \neq 0$.
    \item If $\lambda_k > \lambda_{\varepsilon}$, then there exists $z_k \in \RR$ such that $\lambda_k = f_{\varepsilon}(z_k)$ and $f_{\varepsilon}'(z_k) \neq 0 $ because, if $x_{\varepsilon}$ is such that $f_{\varepsilon}(x_{\varepsilon}) = \lambda_{\varepsilon}$, we have that $f_{\varepsilon}'>0$ on $]x_{\varepsilon}, +\infty[$.
     \item If $\lambda_k = \lambda_{\varepsilon}$, there exists $z_k \in \RR$ such that $\lambda_k = f_{\varepsilon}(z_k)$. Necessarily, $f_{\varepsilon}'(z_k) = 0$ but $f_{\varepsilon}''(z_k) \neq 0$.
\end{itemize}

This shows that the Jordan blocks of $A$ are necessarily of the form 

\begin{equation*}
    \begin{split}
    &(\lambda - f_{\varepsilon}(z_k))^{b_k} \text{ and } (\lambda - f_{\varepsilon}(\bar{z_k}))^{b_k}  ,\, z_k \in \CC -\RR\\
        &(\lambda - f_{\varepsilon}(z_k))^{a_k}   ,\, z_k \in \RR ,\, f_{\varepsilon}(z_k)\neq\lambda_{\varepsilon} \\
         &(\lambda - \lambda_{\varepsilon})^{c_k} \text{ and } (\lambda - \lambda_{\varepsilon})^{c_k}.
    \end{split}
\end{equation*}

Let $Y\in \RR^{d\times d}$ be such that its Jordan blocks are of the form 

\begin{equation*}
    \begin{split}
         & (\lambda - z_k)^{b_k} \text{ and } (\lambda - \bar{z_k})^{b_k}  ,\, z_k \in \CC -\RR, \, f_{\varepsilon}'(z_k)\neq 0 \\
         &(\lambda - z_k)^{a_k}  ,\, z_k \in \RR ,\, f_{\varepsilon}(z_k)\neq\lambda_{\varepsilon}, \,f_{\varepsilon}'(z_k)\neq 0 \\
         &(\lambda - z_k)^{2c_k}  ,\, z_k \in \RR,\, f_{\varepsilon}(z_k) = \lambda_{\varepsilon}.
    \end{split}
\end{equation*}

 Then again by the use of Theorem 7 from \citet{gant} (p. 158), because if $f_{\varepsilon}(z_k) = \lambda_{\varepsilon}$ with $z_k \in \RR$, $f_{\varepsilon}''(z_k) \neq 0$, we have that $f_{\varepsilon}(Y)$ is similar to $A$. 
 Thus $A$ writes $A = P^{-1}f_{\varepsilon}(Y)P = f_{\varepsilon}(P^{-1}YP)$ with $P \in \mathrm{GL}_d(\RR)$. 
 Then, $X = P^{-1}YP$ satisfies $X \in \RR^{d\times d}$ and $f_{\varepsilon}(X) = A$.
\end{proof}

\section{Additional theoretical results}\label{app:additional_results}

\subsection{On the convergence of the solution of a second order model when $\varepsilon \to \infty$}\label{app:prop_eps_infty}
\begin{proposition}[Convergence of the solution when $\varepsilon \xrightarrow{} + \infty$]\label{prop:eps_infty}

We let $x^*$ (resp. $x_{\varepsilon}$) be the solution of $\ddot{x} = f(x,\theta)$ (resp. $\ddot{x} + \frac{1}{\varepsilon} \dot{x} = f(x,\theta)$) on $[0,T]$, with initial conditions $x^*(0) = x_{\varepsilon}(0)=x_0$ and $\dot{x}^*(0) = \dot{x}_{\varepsilon}(0) = v_0$.
Then $x_{\varepsilon}$ converges uniformly to $x^*$ as $\varepsilon \xrightarrow{} + \infty$. 
\end{proposition}

\begin{proof}
The equation $\ddot{x} + \frac{1}{\varepsilon} \dot{x} = f(x,\theta)$ with $x_{\varepsilon}(0)=x_0$, $\dot{x}_{\varepsilon}(0) = v_0$ writes in phase space $(x,v)$
\begin{equation*}
\begin{split}
\begin{cases}
\dot{x} &= v, \quad x(0) = x_0 \\
\dot{v} &= f(x,\theta) - \frac{v}{\varepsilon}, \quad v(0) = v_0.
\end{cases}
\end{split}
\end{equation*}

It then follows from the Cauchy-Lipschitz Theorem with parameters \citep[Theorem 2, Chapter 2]{perko2013differential} that the solutions of this system are continuous in the parameter $\frac1\varepsilon$. That is $x_{\varepsilon}$ converges uniformly to $x^*$ as $\varepsilon \xrightarrow{} + \infty$.

\end{proof}

\subsection{Universality of Momentum ResNets}\label{app:learn_init}

\begin{proposition}[When $v_0$ is free any mapping can be represented]\label{learn_init_speed}

 Consider $h : \RR^d \xrightarrow{} \RR^d$, and the ODE
 \begin{equation*}
     \begin{split}
        \ddot{x} + \dot{x} &= 0\\
        (x(0),\dot{x}(0)) &=(x_0,\frac{h(x_0) - x_0}{1- 1/e})
    \end{split}
 \end{equation*}
 Then $\phi_1(x_0) = h(x_0)$.
\end{proposition}

\begin{proof}
This is because the solution is $\phi_t(x_0) = x_0 - v_0(e^{-t} - 1)$. 
\end{proof}

\subsection{Non-universality of Momentum ResNets when $v_0=0$}

\begin{proposition}[When $v_0 = 0$ there are mappings that cannot be learned if the equation is autonomous.]\label{proposition:limits_init_0}

When $d = 1$, consider the autonomous ODE \begin{equation}\label{eq:limits_init_0}
\begin{split}
\varepsilon \ddot{x} + \dot{x} &= f(x) \\
(x(0),\dot{x}(0)) &= (x_0,0)
\end{split}
\end{equation}
If there exists $x_0 \in \RR+^{*}$ such that $h(x_0) \leq -x_0 $  and $x_0 \leq h(-x_0)$ then $h$ cannot be represented by~\eqref{eq:limits_init_0}.
\end{proposition}
This in particular proves that $x \mapsto \lambda x$ for $\lambda \leq -1$ cannot be represented by this ODE with initial conditions $(x_0,0)$.

 \begin{proof}
 Consider such an $x_0$ and $h$. Since $\phi_1 (x_0) = h(x_0) \leq -x_0$, that $\phi_0(x_0) = x_0$ and that $t \mapsto \phi_t(x_0)$ is continuous, we know that there exists $t_0 \in [0,1]$ such that $\phi_{t_0}(x_0) = -x_0$.
We denote $x(t) = \phi_t(x_0) $, solution of 

$$
     \ddot{x} + \frac{1}{\varepsilon} \dot{x} = f(x)
$$
Since $d = 1$, one can write $f$ as a derivative: $f=-E'$. 
The energy $E_m = \frac{1}{2}\dot{x}^2 + E$ satisfies: 
$$
\dot{E_m} = -\frac{1}{\varepsilon}\dot{x}^2
$$
So that 
$$
E_m(t_0) - E_m(0) = -\frac{1}{\varepsilon} \int_{0}^{t_0}\dot{x}^2
$$
In other words: 
$$
\frac{1}{2}v(t_0)^2 + \frac{1}{\varepsilon} \int_{0}^{t_0}\dot{x}^2 + E(-x_0) = E(x_0)
$$
So that $E(-x_0) \leq E(x_0)$
We now apply the exact same argument to the solution starting at $x_1 = -x_0$. Since $x_0 \leq h(-x_0) = h(x_1)$ there exists $t_1 \in [0,1]$ such that $\phi_{t_1}(x_1) = x_0$. So that:
$$
\frac{1}{2}v(t_1)^2 + \frac{1}{\varepsilon} \int_{0}^{t_1}\dot{x}^2 + E(x_0) = E(-x_0)
$$
So that $E(x_0) \leq E(-x_0)$.
We get that 
$$
E(x_0) = E(-x_0)
$$
This implies that $\dot{x} = 0$ on $[0,t_0]$, so that the first solution is constant and $x_0 = -x_0$ which is absurd because $x_0 \in \RR*$.
\end{proof}
\subsection{When $v_0 = 0$ there are mappings that can be represented by a second-order model but not by a first-order one.}\label{app:prop_lambda}

\begin{proposition}\label{prop:lambda}
There exits $f$ such that the solution of 
$$
     \ddot{x} + \frac{1}{\varepsilon} \dot{x} = f(x)
$$
with initial condition $(x_0,0)$
at time $1$ is
$$
x(1) = -x_0 \times \exp(-\frac{1}{2 \varepsilon})
$$
\end{proposition}

\begin{proof}
Consider the ODE
\begin{equation}
    \ddot{x} + \frac{1}{\varepsilon} \dot{x} = (-\pi^2 - \frac{1}{4\varepsilon^2}) x
\end{equation}
with initial condition $(x_0,0)$
The solution of this ODE is 
$$
x(t) = x_0  e^{-\frac{t}{2 \varepsilon}}(  \cos(\pi  t) + \frac{1}{2\pi \varepsilon}\sin(\pi t))
$$
which at time 1 gives:
$$
x(1) = -x_0 e^{-\frac{1}{2 \varepsilon}}
$$
\end{proof}

\subsection{Orientation preservation of first-order ODEs}\label{app:prop_connected}

\begin{proposition}[The homeomorphisms represented by~\eqref{eq:first_order_ODE} are orientation preserving.]\label{prop:connexe}
If $K \subset \RR^d$ is a compact set and $h : K\xrightarrow{} \RR^d$ is a homeomorphism represented by~\eqref{eq:first_order_ODE}, then $h$ is in the connected component of the identity function on $K$ for the $\|.\|_{\infty}$ topology.
\end{proposition}
We first prove the following: 
\begin{lemma}\label{lemma:compactness}
Consider $K \subset \RR^d$ a compact set. Suppose that $\forall x \in K$, $\Phi_t(x)$ is defined for all $t \in [0,1]$.
Then 
$$
C = \{\Phi_t(x) \mid x \in K, t\in [0,1]\} 
$$
is compact as well.
\end{lemma}
\begin{proof}
We consider $(\Phi_{t_n}(x_n))_{n\in \NN}$ a sequence in $C$. Since $K \times [0,1]$ is compact, we can extract sub sequences $(t_{\phi(n)})_{n\in \NN}$, $(x_{\phi(n)})_{n\in \NN}$ that converge respectively to $t_0$ and $x_0$. We denote them $(t_n)_{n\in \NN}$ and $(x_n)_{n\in \NN}$ again for simplicity of the notations.
We have that:
$$
\|\Phi_{t_n}(x_n)-\Phi_{t}(x)\| \leq \|\Phi_{t_n}(x_n)-\Phi_{t_n}(x)\| + \|\Phi_{t_n}(x)-\Phi_{t}(x)\|.
$$
Thanks to Gronwall's lemma, we have
$$
\|\Phi_{t_n}(x_n)-\Phi_{t_n}(x)\| \leq \|x_n-x\|\exp{(k t_n)},
$$
where $k$ is $f$'s Lipschitz constant. So that $\|\Phi_{t_n}(x_n)-\Phi_{t_n}(x)\|\xrightarrow{} 0$ as $n \xrightarrow{} \infty$.
In addition, it is obvious that $\|\Phi_{t_n}(x)-\Phi_{t}(x)\| \xrightarrow{} 0$ as $n \xrightarrow{} \infty$.
We conclude that
$$
\Phi_{t_n}(x_n) \xrightarrow{} \Phi_{t}(x) \in C,
$$
so that $C$ is compact.
\end{proof}

\begin{proof}
Let's denote by $H$ the set of homeomorphisms defined on $K$. The application 
$$
\Psi : [0,1] \xrightarrow{} H 
$$
defined by
$$
\Psi(t) = \Phi_t
$$
is continuous. Indeed, we have for any $x_0$ in $\RR^d$ that 
$$
\|\Phi_{t+\varepsilon}(x_0) - \Phi_{t}(x_0)\| = \|\int_{t}^{t+\varepsilon}f(\Phi_{s}(x_0))ds \| \leq \varepsilon M_f,
$$
where $M_f$ bounds the continuous function $f$ on $C$ defined in lemma~\ref{lemma:compactness}.
Since $M_f$ does not depend on $x_0$, we have that 
$$
\|\Phi_{t+\varepsilon} - \Phi_{t}\|_{\infty} \xrightarrow{} 0
$$
as $\varepsilon \xrightarrow{} 0$, which proves that $\Psi$ is continuous.
Since $\Psi(0) = Id_K$, we get that $\forall t \in [0,1]$, $\Phi_t$ is connected to $Id_K$.
\end{proof}

\subsection{On the linear mappings represented by autonomous first order ODEs in dimension $1$}

Consider the autonomous ODE
\begin{equation}\label{eq:autonomous}
    \dot{x} = f(x),
\end{equation}

\begin{theorem}[Linearity]\label{theo1}
Suppose $d=1$. If~\eqref{eq:autonomous} represents a linear mapping $x \mapsto ax$ at time $1$, we have that $f$ is linear.
\end{theorem}
\begin{proof}
If $a=1$, consider some $x_0 \in \RR$. Since $\Phi_1(x_0) = x_0 = \Phi_0(x_0)$, there exists, by Rolle's Theorem a $t_0 \in [0,1]$ such that $\dot{x}(t_0) = 0$. Then $f(x(t_0)) = 0$. But since the constant solution $y = x(t_0)$ then solves $\dot{y} = f(y), y(0)= x(t_0)$, we get by the unicity of the solutions that $x(t_0) = y(0) = x(1) = y(1-t_0) = x_0$. So that $f(x_0) = f(x(t_0)) = 0$. Since this is true for all $x_0$, we get that $f=0$. We now consider the case where $a \ne 1$ and $a > 0$.
Consider some $x_0 \in \RR^{*}$. If $f(x_0) = 0$, then the solution constant to $x_0$ solves~\eqref{theo1}, and thus cannot reach $ax_0$ at time $1$ because $a \neq 1$. Thus, $f(x_0) \ne 0$ if $x_0 \ne 0$. Second, if the trajectory starting at $x_0 \in \RR^{*}$ crosses $0$ and $f(0) = 0$, then by the same argument we know that $x_0 = 0$, which is absurd. So that, $\forall x_0 \in \RR^{*}$, $\forall t \in [0,1]$, $f(\Phi_t(x_0)) \ne 0$ .
We can thus rewrite~\eqref{theo1} as 
\begin{equation}\label{inver}
    \frac{\dot{x}}{f(x)} = 1.
\end{equation}
Consider $F$ a primitive of $\frac{1}{f}$.
Integrating~\eqref{inver}, we get
$$
F(ax_0) - F(x_0) = \int_{0}^{1} F'(x(t))\dot{x}(t)\mathrm{d} t = 1.
$$
In other words, $\forall x \in \RR*$:
$$
F(ax) = F(x) + 1.
$$
We derive this equation and get:
$$
af(x) = f(ax).
$$
This proves that $f(0) = 0$.
We now suppose that $a>1$.
We also have that 
$$
a^n f(\frac{x}{a^n}) =  f(x).
$$
But when $n \xrightarrow{} \infty$, $f(\frac{x}{a^n}) = \frac{x}{a^n}f'(0) + o(\frac{1}{a^n})$ so that 
$$
f(x) = f'(0)x
$$
and $f$ is linear.
The case $a < 1$ treats similarly by changing $a^n$ to $a^{-n}$.
\end{proof}

\subsection{There are mappings that are connected to the identity that cannot be represented by a first order autonomous ODE}

In bigger dimension, we can exhibit a matrix in $\mathrm{GL}_d^{+}(\RR)$ (and hence connected to the identity) that cannot be represented by the autonomous ODE~\eqref{eq:autonomous}.
\begin{proposition}[A non-representable matrix]\label{prop:contre_ex}
Consider the matrix 
$$
A=
\begin{pmatrix}
-1 & 0  \\
0 & -\lambda \\
\end{pmatrix},
$$ where $\lambda > 0$ and $\lambda \neq 1$.
Then $A \in GL_2^{+}(\RR)  - GL_2(\RR)^{2}$ and $A$ cannot be represented by~\eqref{eq:autonomous}.
\end{proposition}
\begin{proof}
The fact that $A \in GL_2^{+}(\RR)  - GL_2(\RR)^{2}$ is because $A$ has two single negative eigenvalues, and because $\det(A) = \lambda > 0$.
We consider the point $(0,1)$. At time $1$, it has to be in $(0,-\lambda)$. Because the trajectory are continuous, there exists $0 < t_0 < 1$ such that the trajectory is at $(x,0)$ at time $t_0$, and thus at $(-x,0)$ at time $t_0 + 1$, and again at $(x,0)$ at time $t_0+2$. However, the particle is at $(0,\lambda^2)$ at time $2$. All of this is true because the equation is autonomous. Now, we showed that trajectories starting at $(0,1)$ and $(0,\lambda^2)$ would intersect at time $t_0$ at $(x,0)$, which is absurd. Figure~\ref{fig:contre_ex} illustrates the paradox.
\end{proof}

\begin{figure}[ht]
\begin{center}
\centerline{\includegraphics[width=0.6\columnwidth]{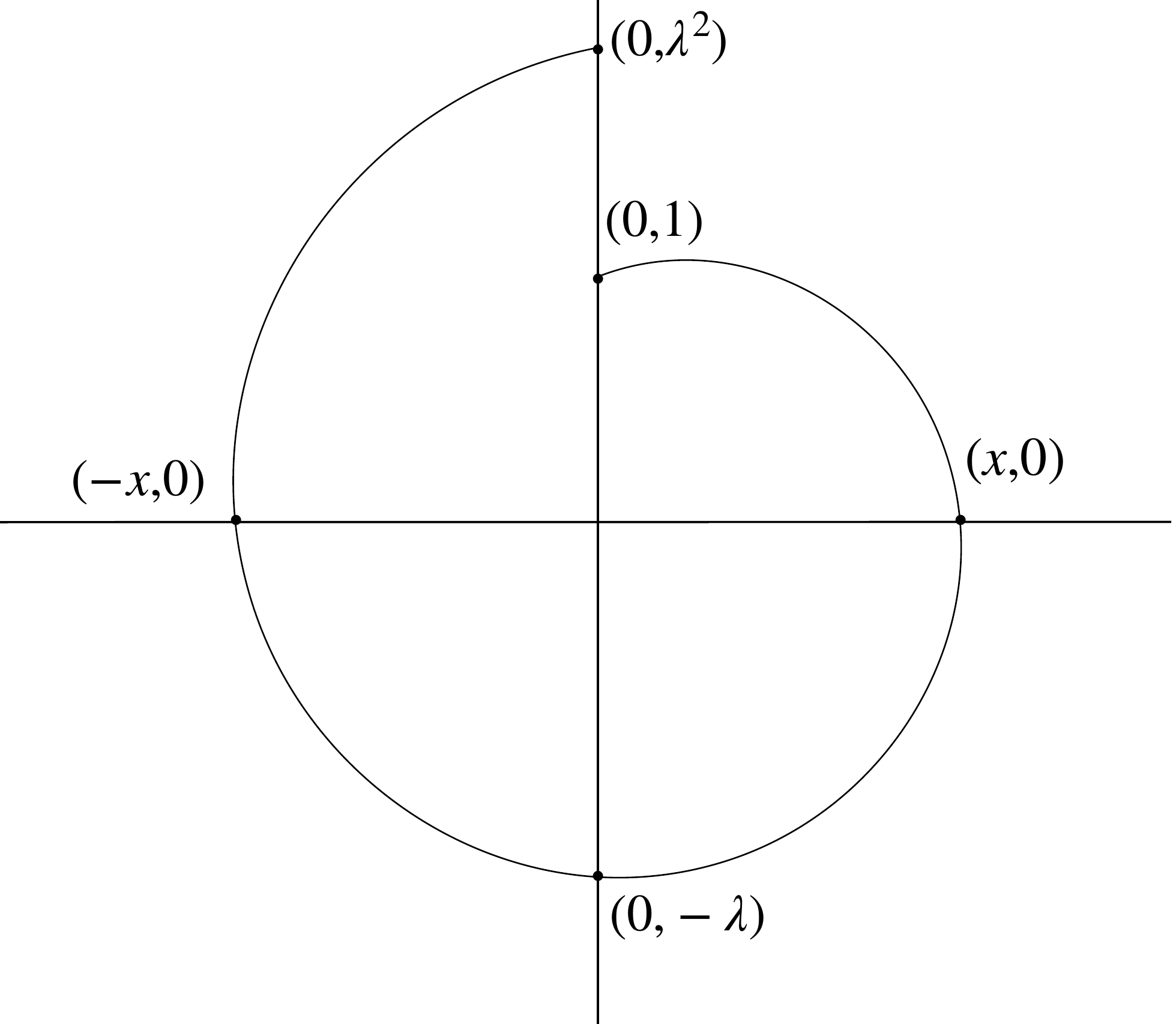}}
\caption{Illustration of Proposition~\ref{prop:contre_ex}. The points starting at $(0,1)$ and $(0,\lambda^2)$ are distinct but their associated trajectories would have to intersect in $(x,0)$, which is impossible.}
\label{fig:contre_ex}
\end{center}
\end{figure}

\section{Exact multiplication}\label{app:memory_savings}

\begin{algorithm}
   \caption{Exactly reversible multiplication by a ratio, from \citet{10.5555/3045118.3045343}}
   \label{alg:reversible-mult}
\begin{algorithmic}[1]
   \STATE {\bfseries Input:} Information buffer $i$, value $c$, ratio $n / d$  
   \STATE $i = i \times d$ 
   \STATE $i = i + (c \! \mod d)$  \label{step:f2}
   \STATE $c = c \div d$                \label{step:f3}
   \STATE $c = c \times n$         \label{step:b1}
   \STATE $c = c +  (i \! \mod n)$         \label{step:b2}
   \STATE $i = i \div n$          \label{step:b3}
   \STATE \textbf{return} updated buffer $i$, updated value $c$
\end{algorithmic}
\end{algorithm}

We here present the algorithm from~\citet{10.5555/3045118.3045343}. In their paper, the authors represent $\gamma$ as a rational number, $\gamma = \frac{n}{d} \in \QQ$. The information is lost during the integer division of $v_{n}$ by $d$ in~\eqref{eq:Momentum ResNet}. The store this information, it is sufficient to store the remainder $r$ of this integer division. $r$ is stored in an “information buffer” $i$. To update $i$, one has to left-shift the bits in $i$ by multiplying it by $n$ before adding $r$. The entire procedure is illustrated in Algorithm~\ref{alg:reversible-mult} from~\citet{10.5555/3045118.3045343}.

\section{Experiment details}\label{app:experiment_details}

In all our image experiments, we use Nvidia Tesla V100 GPUs.
\\
\\
For our experiments on CIFAR-10 and 100, we used a batch-size of $128$ and we employed SGD with a momentum of $0.9$. The training was done over $220$ epochs. The initial learning rate was $0.01$ and was decayed by a factor $10$ at epoch $180$. A constant weight decay was set to $5 \times 10^{-4}$. Standard inputs preprocessing as proposed in Pytorch \citep{paszke2017automatic} was performed.  
\\
\\
For our experiments on ImageNet, we used a batch-size of $256$ and we employed SGD with a momentum of $0.9$. The training was done over $100$ epochs. The initial learning rate was $0.1$ and was decayed by a factor $10$ every $30$ epochs. A constant weight decay was set to $10^{-4}$. Standard inputs preprocessing as proposed in Pytorch \citep{paszke2017automatic} was performed: normalization, random croping of size $224 \times 224$ pixels, random horizontal flip. 
\\
\\
For our experiments in the continuous framework, we adapted the code made available by \citet{chen2018neural} to work on the CIFAR-10 data set and to solve second order ODEs. We used a batch-size of $128$, and used SGD with a momentum of $0.9$. The initial learning rate was set to $0.1$ and reduced by a factor $10$ at iteration $60$. The training was done over $120$ epochs. 
\\
\\
For the learning to optimize experiment, we generate a random Gaussian matrix $D$ of size $16\times 32$. The columns are then normalized to unit variance.
We train the networks by stochastic gradient descent for $10000$ iterations, with a batch-size of $1000$ and a learning rate of $0.001$.
The samples $y_q$ are generated as follows:
we first sample a random Gaussian vector $\tilde{y}_q$, and then we use $y_q = \frac{\tilde{y}_q}{\|D^{\top}\tilde{y}_q\|_{\infty}}$, which ensures that every sample verify $\|D^{\top}y_q\|_{\infty} = 1$. This way, we know that the solution $x^*$ is zero if and only if $\lambda \geq 1$. The regularization is set to $\lambda = 0.1$.
\section{Backpropagation for Momentum ResNets}\label{app:backprop_mom_nets}

In order to backpropagate the gradient of some loss in a Momentum ResNet, we need to formulate an explicit version of~\eqref{eq:Momentum ResNet}.
Indeed,~\eqref{eq:Momentum ResNet} writes explicitly 
\begin{equation}
\begin{split}
  v_{n+1} & = \gamma v_n + (1-\gamma) f(x_n,\theta_n)\\ 
  x_{n+1} & = x_n + (\gamma v_n + (1-\gamma) f(x_n,\theta_n)).
\end{split}
\end{equation}
Writing $z=(x,v)$, the backpropagation for Momentum ResNets then writes, for some loss $L$ 
$$
\nabla_{z_{k-1}}L=  \left[ {\begin{array}{cc}
  I +  (1-\gamma) \partial_xf(x_{k-1},\theta_{k-1}) &  \gamma I\\
   (1-\gamma) \partial_xf(x_{k-1},\theta_{k-1}) & \gamma I \\
  \end{array} } \right]^{T} \nabla_{z_{k}}L
$$
$$
\nabla_{\theta_{k-1}}L = (1 - \gamma) \left[ {\begin{array}{cc}
 \partial_{\theta}f(x_{k-1},\theta_{k-1}) \\
   \partial_{\theta}f(x_{k-1},\theta_{k-1})  \\
  \end{array} } \right]^{T} \nabla_{z_{k}}L.
$$
We implement these formula to obtain a custom Jacobian-vector product in Pytorch.

\section{Additional figures}\label{app:figures}

\subsection{Learning curves on CIFAR-10}

We here show the learning curves when training a ResNet-101 and a Momentum ResNet-101 on CIFAR-10.

\begin{figure}[H]
\vskip 0.2in
\begin{center}
\centerline{\includegraphics[width=\columnwidth]{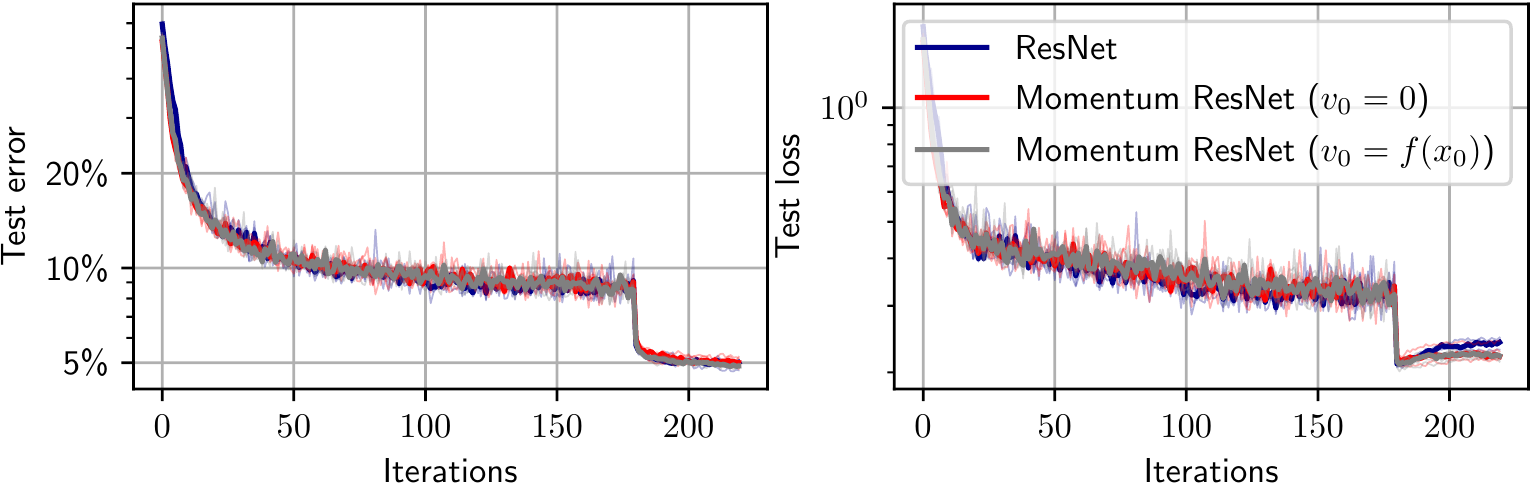}}
\caption{Test error and test loss as a function of depth on CIFAR-10 with a ResNet-101 and two Momentum ResNets-101.}
\end{center}
\vskip -0.2in
\end{figure}

\end{document}